\documentclass[opre,nonblindrev]{informs3_hide}

\DoubleSpacedXI 

\usepackage{endnotes}
\let\footnote=\endnote
\let\enotesize=\normalsize
\def\notesname{Endnotes}%
\def\makeenmark{\hbox to1.275em{\theenmark.\enskip\hss}}
\def\enoteformat{\rightskip0pt\leftskip0pt\parindent=1.275em
  \leavevmode\llap{\makeenmark}}

\usepackage{natbib}
 \bibpunct[, ]{(}{)}{,}{a}{}{,}%
 \def\bibfont{\small}%
\TheoremsNumberedThrough     
\ECRepeatTheorems

\EquationsNumberedThrough    

\newcommand{\xhdr}[1]{\vspace{1mm} \noindent{\bf #1}}
\usepackage{bbm}
\usepackage[ruled]{algorithm2e}
\usepackage{subcaption}
\newcommand{\HS}{\texttt{H-SF}}
\newcommand{\AR}{\mathcal A_{\text{random}}}
\begin{document}


\RUNAUTHOR{xxx}

\RUNTITLE{Adaptively Robust LLM Inference Optimization under Prediction Uncertainty}

\TITLE{Adaptively Robust LLM Inference Optimization under Prediction Uncertainty}

\ARTICLEAUTHORS{%
\AUTHOR{Zixi Chen}
\AFF{Department of Mathematics, Peking University,  \EMAIL{chenzixi22@stu.pku.edu.cn}} 
\AUTHOR{Yinyu Ye}
\AFF{Department of Management Science and Engineering, Stanford University \&  Department of Industrial Engineering and Decision Analytics, HKUST, \EMAIL{yyye@stanford.edu}}
\AUTHOR{Zijie Zhou}
\AFF{Department of Industrial Engineering and Decision Analytics, HKUST, \EMAIL{jerryzhou@ust.hk}}
} 

\ABSTRACT{
We study the problem of optimizing Large Language Model (LLM) inference scheduling to minimize total latency. LLM inference is an online and multi-task service process and also heavily energy consuming by which a pre-trained LLM processes input requests and generates output tokens sequentially. Therefore, it is vital to improve its scheduling efficiency and reduce the power consumption while a great amount of prompt requests are arriving. A key challenge in LLM inference scheduling is that while the prompt length is known upon arrival, the output length, which critically impacts memory usage and processing time, is unknown. To address this uncertainty, we propose algorithms that leverage machine learning to predict output lengths, assuming the prediction provides an interval classification (min-max range) for each request.  

We first design a conservative algorithm, \(\mathcal{A}_{\max}\), which schedules requests based on the upper bound of predicted output lengths to prevent memory overflow. However, this approach is overly conservative: as prediction accuracy decreases, performance degrades significantly due to potential overestimation. To overcome this limitation, we propose \(\mathcal{A}_{\min}\), an adaptive algorithm that initially treats the predicted lower bound as the output length and dynamically refines this estimate during inferencing. We prove that \(\mathcal{A}_{\min}\) achieves a log-scale competitive ratio. Through numerical simulations, we demonstrate that \(\mathcal{A}_{\min}\) often performs nearly as well as the hindsight scheduler, highlighting both its efficiency and robustness in practical scenarios. Moreover,  \(\mathcal{A}_{\min}\) relies solely on the lower bound of the prediction interval—an advantageous design choice since upper bounds on output length are typically more challenging to predict accurately.
}

\maketitle

\section{Introduction}

Recent advances in large language models (LLMs) \citep{brown2020language,chowdhery2023palm,openai2023gpt,kaplan2020scaling,wei2022emergent} have redefined the boundaries of artificial intelligence, showcasing exceptional proficiency in generating human-like text across a wide spectrum of languages and contexts. These advanced neural architectures, built upon vast and diverse textual datasets, have become integral to a wide range of practical applications. Their influence extends to AI-driven conversational interfaces \citep{anthropic2023claude,characterai2023,chatgpt2023,openai2023gpt}, next-generation search technologies \citep{bing,googlebard,perplexity}, intelligent programming assistants \citep{codewhisperer2023,githubcopilot2023}, and operations service management \citep{huang2025orlm,cascella2023evaluating,sallam2023utility}, highlighting their adaptability across both specialized and everyday use cases.

Large Language Model (LLM) inference—the process by which a pre-trained LLM processes input requests and generates output tokens sequentially—is not just a computational task but a critical operational challenge with real-world cost and energy implications. Optimizing this process can significantly reduce infrastructure expenses and power consumption, especially given the scale of modern LLM deployments. This problem sits at the intersection of AI systems and operations research (OR), where classical scheduling techniques, adapted to LLM-specific constraints, can unlock substantial efficiency gains.

For each request, the inference consists of two key phases:
\begin{enumerate}
    \item \textit{Input (Prefill) Phase: } The LLM reads the input prompt, computes the query, key, and value representations for each token in the prompt, and produces the first output token.
    \item \textit{Output (Decode) Phase: } The LLM generates subsequent tokens autoregressively—using all previously generated tokens as context—until completion.
\end{enumerate}
For example, given the input prompt ``KMeans is used", the LLM might first generate the token ``for", then use ``KMeans is used for" to produce ``clustering", and so on. Crucially, while the input length is known upfront, the output length is inherently unknown until generation finishes. Recent work has explored methods to predict output lengths in advance, often employing auxiliary machine learning models or even LLMs themselves (\cite{zheng2023response,qiu2024efficient,fu2024efficient,chen2024kvdirect, shahout2024don}). Most of these approaches typically classify requests into predefined output-length intervals based on their predictions.

Given the high volume of requests we typically handle, simultaneous processing is often unfeasible. To optimize efficiency, it is essential to design scheduling algorithms using operations research techniques to prioritize request sequences. Accurate output length prediction is critical for efficient LLM inference scheduling, especially when handling large volumes of requests. Such predictions enable effective scheduling policies by providing two key insights: (i) \textbf{Execution Time Estimation: } Since tokens are generated sequentially, the output length directly correlates with the request’s completion time. (ii) \textbf{Memory Usage Forecasting: } The key-value (KV) cache memory grows linearly during decoding, as each new token’s KV embeddings are stored and retained for all subsequent steps. Thus, predicting output length allows systems to anticipate both the computational duration and memory footprint of each request.

The design of efficient scheduling policies for LLM inference—which determines the processing order of a great number of input prompts under output length predictions—has been explored under idealized assumptions. For instance, when predictions are assumed to be perfectly accurate, prior work \cite{jaillet2025online,shahout2024don} proposes shortest-first algorithms that prioritize requests with smaller predicted output lengths. This approach offers two advantages: (i) shorter jobs finish faster, reducing the total waiting time; (ii) the system can process more requests concurrently in larger batches at early stage since shorter requests consume less KV cache memory. However, real-world predictions are inherently imperfect. Inaccurate predictions introduce two key challenges:
\begin{itemize}
    \item \textbf{Suboptimal Scheduling: } Misclassify a long output as short disrupts the processing order, delaying other requests.
    \item \textbf{Memory Overflows and Underflows: } Underestimated memory demands can exceed GPU KV cache limits, forcing request cancellations and halting execution. Overestimated memory demands can reduce the concurrency, increasing the total waiting time.
\end{itemize}

This paper rigorously addresses these challenges by developing robust scheduling algorithms grounded in mathematical theory, ensuring efficiency even under prediction uncertainty.

\subsection{Main Contribution}

We introduce the main contribution and the outline in this section. 

\xhdr{Naive Conservative Algorithm $\mathcal{A}_{\max}$ and its Competitive Analysis. } In Section~\ref{sec:benchmark}, we propose a naive benchmark algorithm, $\mathcal{A}_{\max}$, that conservatively treats the predicted upper bound of each job’s output length as its true length. This approach guarantees no memory overflows by over-provisioning KV cache resources for all requests. However, by overestimating the memory usage for each request, the algorithm hurts the concurrency if the prediction error is large. Our analysis reveals two key results: (i) We construct an adversarial arrival sequence proving that $\mathcal{A}_{\max}$ achieves a competitive ratio of at least \( \frac{\alpha^{-1} (1+\alpha^{-1/2})}{2}\), where $\alpha$ is the ratio of the prediction interval’s lower bound to its upper bound. (ii) Moreover, we rigorously show the upper bound of the competitive ratio of $\mathcal{A}_{\max}$ is \( \frac{\alpha^{-1}(1 + \alpha^{-1})}{2} \). The proof leverages a novel memory-preserving combinatorial technique (Section~\ref{sec:memory}), which may inspire further research in LLM inference scheduling under prediction uncertainty.

\xhdr{Robust Advanced Algorithm $\mathcal{A}_{\min}$ and its Competitive Analysis. } 
Motivated by the observation that the competitive ratio of $\mathcal{A}_{\max}$ becomes unbounded as $\alpha$ shrinks (i.e., when predictions lack confidence), we design $\mathcal{A}_{\min}$, a robust algorithm that leverages only the lower prediction bound. Initially, $\mathcal{A}_{\min}$ underestimates each request's output length as its lower bound, dynamically refining this estimate during execution. In Section~\ref{sec:algrobust}, we prove that $\mathcal{A}_{\min}$ achieves asymptotic optimality with a competitive ratio of $\mathcal{O}(\log(\alpha^{-1}))$, demonstrating significantly stronger robustness than $\mathcal{A}_{\max}$. The analysis in Subsection \ref{subsec:log} is notably more challenging than that of $\mathcal{A}_{\max}$: we derive the competitive ratio's closed form as a Rayleigh quotient and then develop a novel estimation technique to bound this quotient logarithmically. In Section \ref{sec:extension}, we also study the performance of $\mathcal{A}_{\min}$ over some specific output distributions. We show that under two-point distribution, geometric distribution, linear weighted geometric distribution, the competitive ratios of $\mathcal{A}_{\min}$ is uniformly bounded by 1.5, 1.7, and 1.56 respectively.

\xhdr{Numerical Experiments. } In Section~\ref{sec:num}, we evaluate our algorithms through three sets of experiments on a real-world dataset~\cite{zheng2023lmsys}. First, we establish a baseline where all requests share an identical prediction interval, representing perfect prediction uniformity. Second, we implement a binned configuration where requests are categorized into ten discrete prediction intervals based on their predictions. Third, we evaluate a fully individualized setting where each request is assigned its own unique prediction interval centered around the true output length, better reflecting practical scenarios where prediction confidence varies per request. Notably, while the second experiment groups requests by interval similarity, the third preserves each request's distinct interval characteristics. Across all three configurations, $\mathcal{A}_{\max}$ demonstrates strong performance only when predictions are highly accurate, while $\mathcal{A}_{\min}$ proves remarkably robust, consistently matching or approaching the performance of the hindsight algorithm that operates with perfect knowledge of the true output lengths.

\subsection{Other Related Work}

\xhdr{Optimization in LLM Inference. } While existing LLM inference literature has predominantly focused on system design and engineering optimizations \citet{patel2023splitwise, zhong2024distserve, yu2022orca, agrawal2023sarathi, agrawal2024taming}, recent work has begun establishing theoretical foundations for LLM inference scheduling. The seminal work of \cite{jaillet2025online} first formalized a mathematical framework for analyzing algorithm performance in this context. Our work extends this foundation by introducing prediction intervals—where algorithms must operate robustly without knowledge of true output lengths. Parallel developments include \cite{ao2025optimizing}, which incorporates multi-metric scheduling constraints, \cite{wang2025llm}, which designs efficient scheduling algorithms for the setting where the input size is heterogeneous, and \cite{li2025throughput}, which characterizes stability conditions for LLM inference systems. Together, these works mark the emergence of theoretical optimization approaches complementing the field's traditional systems focus.

\xhdr{Decision-Making and Operations Management with Prediction Uncertainty. } Our work connects to broader research on learning-augmented algorithms, where predictions inform decision-making under uncertainty. While \cite{shahout2024don} develops prediction methods for LLM inference without considering robust scheduling, other domains have deeply explored this paradigm: \cite{lykouris2021competitive} analyzes online caching with access predictions, \cite{golrezaei2023online} examines online resource allocation with predicted demand, \cite{agrawal2011unified} studies the prediction market design, and \cite{jin2022online,antoniadis2020secretary} investigates prediction-enhanced secretary problem and online matching. The robust optimization literature \citep{bertsimas2004price} also provides foundational techniques for handling uncertainty, while recent advances by \cite{zhou2022advance} and \cite{zhang2023robust} demonstrate applications to scheduling problems. Our work bridges these perspectives by developing robust policies specifically for prediction-informed LLM scheduling frameworks.

\xhdr{Online and Offline Scheduling.} The operations research community has extensively studied online and offline scheduling problems \citep{chen1998review,blazewicz2007handbook,mak2015appointment,leung2010competitive}. These problems involve determining the processing order and start times for a large set of jobs that cannot be executed simultaneously. Prior work includes batch scheduling, where jobs are processed in parallel or grouped into batches \citep{xing2000parallel,yang2003approximation,cao2005parallel,chen2008logistics}, and resource-constrained scheduling, where jobs consume limited resources during execution \citep{brucker1999resource,hiermann2015metaheuristics}.

The LLM inference scheduling problem studied in our work combines these two dimensions: it is both batch-based and resource-constrained, with GPU memory acting as a reusable resource. Crucially, the memory usage patterns of LLM inference exhibit unique characteristics that distinguish our problem from classical scheduling models. These distinct features necessitate the development of new algorithmic approaches tailored to this setting.

\section{Problem Settings}

In this paper, we study an extension of the LLM inference scheduling model proposed in \cite{jaillet2025online}. We begin by reviewing the mathematical framework presented in \cite{jaillet2025online} and then describe the modifications introduced in our model. We consider a setting with a single computational worker equipped with a KV cache of size \( M > 0 \), capable of storing up to \( M \) tokens. In practice, the value of \( M \) depends on the complexity of the large language model and the available hardware memory of the computational worker. We assume that \( M \) is known to the decision-maker. A total of \( n \) jobs (prompts) are waiting in the queue, where each job \( i \) has a size \( s_i > 0 \), representing the number of tokens in the prompt. Following the assumptions in \cite{jaillet2025online}, and motivated by the observation that in real-world parallel computing environments prompts assigned to a worker typically have similar lengths, we assume \( s_i = s > 0 \) for all \( i \in [n] \). The value of \( s \) is known to the decision-maker.

\xhdr{Prediction Interval for Output Length. } To process each prompt, the model generates the output token by token. 
Let \( o_i > 0 \) denote the realized output length of job \( i \), representing the number of tokens generated in its response. Since the realized value \( o_i \) is not revealed until the last output token is generated, it remains unknown to the decision-maker during the process. 
In \cite{jaillet2025online}, it is assumed that the realized output length \( o_i \) is perfectly predicted and known in advance. In contrast, our current model assumes that \( o_i \) is unknown a priori but can be predicted to lie within an interval \([ \ell, u ]\), and we denote $\alpha = \frac{\ell}{u}$. For clarity and to build intuition regarding the complexity of the problem and the design of our algorithm, we first consider the case where all jobs share the same predicted interval \([ \ell, u ]\). We then generalize the result in Section \ref{subsec:hetero} to allow each job’s output length to fall within different predicted intervals.


\xhdr{Batch Processing, Memory Constraint, and Cancellation of Requests. }  
Following the model in \cite{jaillet2025online}, we allow jobs to be processed in batches. At each discrete time step \( t \), the scheduler selects a subset of jobs \( S^{(t)} \) to form a batch for processing. Each batch takes exactly one unit of time to process. For any job \( i \in S^{(t)} \), let \( a_i^{(t)} \in \{0, 1, \ldots, o_i\} \) denote the number of output tokens it has generated by time \( t \). If \( a_i^{(t)} < o_i \), then processing the batch at time \( t \) results in the generation of the \( (a_i^{(t)} + 1) \)-th token for job \( i \). Once \( a_i^{(t)} = o_i \), job \( i \) is fully completed.

Each batch must also satisfy a memory constraint determined by the KV cache limit \( M \). At any time \( t \), let \( \mathcal{A}^{(t)} \) denote the set of active jobs—i.e., those that have begun processing in some earlier batch. For each active job \( i \in \mathcal{A}^{(t)} \), the memory required is \( s + a_i^{(t)} \), accounting for both the prompt size and the number of output tokens generated so far. The total memory usage at time \( t \) must not exceed the limit \( M \), yielding the following constraint:
\begin{equation} \label{eq:constraintM}
\sum_{i \in \mathcal A^{(t)}} (s + a_i^{(t)}) \leq M, \qquad \forall t \geq 0.
\end{equation}

Since our setting assumes that only interval predictions of \( o_i \) are available, it is hard to guarantee feasibility of constraint \eqref{eq:constraintM} under some non-preemptive policies. For example, suppose the total memory usage at time \( t \) is \( M - 1 \) with two active jobs, each with at least two tokens remaining. Even if the decision-maker delays one job and processes the other, both jobs will eventually require at least 2 units of additional memory. Thus, there is no way to complete both without exceeding the memory constraint in some future time step.

To address this challenge, based on the non-preemptive structure, our model permits the cancellation of jobs. That is, the scheduler may cancel any active job \( i \) at time \( t \), discarding all previously generated output tokens \( a_i^{(t)} \) and resetting the job’s state to unprocessed. We emphasize that since the sampling method is greedy, the total realized output length \( o_i \) of each job remains fixed, regardless of how many times it is cancelled and restarted.

\xhdr{Evaluation Metrics.}  
We use total end-to-end latency as our performance metric. Denote the vector $\mathbf{o}=(o_1,o_2,\ldots,o_n)$ which contains the true value of output length of all requests. W.L.O.G, we assume that $o_1 \leq o_2 \leq \ldots \leq o_n$. Under a scheduling policy \( \pi \), all prompts are arranged into an input queue \( I = ( I_0, I_1, \cdots, I_T ) \), where \( I_t \) denotes the set of prompts whose last (and final) starting time is at time \( t \)—that is, they are not cancelled after \( t \). For each request \( i \in I_t \), we define its latency as \( L_i = t + o_i \), which corresponds to the completion time of its last output token. Since all jobs arrive at time \( t = 0 \), this is also the total end-to-end latency of request \( i \), and the total end-to-end latency of the system is then given by
\begin{align}
\label{eq:latency}
    \text{TEL}(\mathbf o;\pi) := \sum_{i \in [n]} L_i = \sum_{t = 0}^{T} t \cdot |I_t| + \sum_{i \in [n]} o_i.
\end{align}
The second equality follows by grouping all jobs according to their final starting time: each job in \( I_t \) contributes latency \( t \) from waiting and \( o_i \) from processing. The first term sums all waiting times, while the second term sums all output lengths.

\section{Benchmark Algorithms} \label{sec:benchmark}

In this section, we introduce several benchmark algorithms to evaluate a given scheduling policy \( \pi \). A central challenge in our model is that the decision-maker does not know the exact output length \( o_i \) of each request \( i \) in advance. Instead, the only information available is that \( o_i \in [\ell, u] \), and we define the uncertainty parameter as \( \alpha = \frac{\ell}{u} \). The true value of each \( o_i \) is selected adversarially from this interval at the outset. To assess performance under this uncertainty, Subsection \ref{subsec:hindsight} introduces a hindsight benchmark algorithm from \cite{jaillet2025online}, in which all output lengths \( o_i \) are known ahead of time. We also define the competitive ratio as a metric for evaluating policies that operate without access to the true values of \( o_i \). In Subsection \ref{subsec:naive}, we propose a naive benchmark algorithm that lacks knowledge of \( o_i \), and analyze its competitive ratio.

\subsection{Review of Hindsight Benchmark Algorithm and Competitive Ratio} \label{subsec:hindsight}

In the hindsight setting—where the decision-maker has complete knowledge of the output lengths \( o_i \) for all \( i \in [n] \)—our model closely aligns with the setting in \cite{jaillet2025online}, with one key distinction: we allow for job cancellations. However, when the exact values of \( o_i \) are known in advance, cancellations are unnecessary. This is because the memory usage over time can be precisely anticipated and managed using the approach introduced in \cite{jaillet2025online} as follows: at each time \( t \), let \( R_t \) be the set of pending requests awaiting processing. Suppose we consider adding a subset \( U \subset R_t \) to the batch. Define \( t_{\max}(U) := \max_{i \in U} \{ t + o_i \} \), which represents the latest time any job in \( U \) will complete if processing starts at time \( t \). To ensure that the memory constraint is respected throughout this interval, we must verify that the total memory consumption remains below the KV cache limit \( M \) at every time \( t' \in [t, t_{\max}(U)] \). This requirement is formalized by the following constraint:
\begin{equation} \label{eqn:Constraint}
\sum_{i \in S^{(t)}} (s + t' - p_i) \cdot \mathbbm{1}_{\{ o_i \geq t' - p_i \}} + \sum_{i \in U} (s + t' - t) \cdot \mathbbm{1}_{\{ o_i \geq t' - t \}} \leq M, \quad \forall t' \in [t, t_{\max}(U)],
\end{equation}
where \( S^{(t)} \) denotes the set of jobs already in progress at time \( t \), and \( p_i \) is the last start time of job \( i \). The first summation captures memory usage from ongoing jobs, while the second accounts for the new jobs in \( U \). As long as this constraint is satisfied for all relevant time points, the batch is guaranteed to respect the memory limit and no cancellations are needed.

Furthermore, when the decision-maker has full knowledge of the values \( o_i \), either in theory (\cite{jaillet2025online}) or in practical implementations (\cite{shahout2024don,zheng2023response,qiu2024efficient,fu2024efficient,chen2024kvdirect}), one effective batching strategy is shortest-job-first. The classical OR heuristic attempts to include as many of the shortest remaining jobs with respect to $o_i$ as possible in each batch without violating constraint \eqref{eqn:Constraint}. The rationale is twofold: first, shorter jobs contribute less to waiting time, reducing overall latency; second, since the peak memory usage of a job is \( s + o_i \), shorter jobs also occupy less memory, allowing more requests to be packed into earlier batches. We refer to this hindsight benchmark as (Hindsight-Shortest First) \HS, and describe it formally in Appendix \ref{append:benchmark}. Next, we describe the definition of competitive ratio:

\begin{definition}[Competitive Ratio] \label{def:cr}  
Let \( \mathbf{o} = (o_1, o_2, \ldots, o_n) \in [\ell, u]^n \) denote the vector of true output lengths for all requests. For a scheduling policy \( \pi \), let \( \text{TEL}(\mathbf{o};\pi) \) denote the total end-to-end latency under policy \( \pi \), and let \( \text{TEL}(\mathbf{o};\HS) \) denote the latency achieved by \HS which has full knowledge of $\mathbf{o}$. Then, the competitive ratio of policy \( \pi \) is defined as:
\begin{align*}
    \text{CR}(\pi) := \sup_{\mathbf{o} \in [\ell, u]^n} \frac{\mathbb{E}[\text{TEL}(\mathbf{o};\pi)]}{\text{TEL}(\mathbf{o};\HS)}.
\end{align*}
\end{definition}

\subsection{Naive Benchmark Algorithm} \label{subsec:naive}

Now consider the setting where the decision-maker only has access to the prediction interval \( [\ell, u] \) for each job. Since all intervals are identical, the decision-maker cannot distinguish between jobs and possesses no information to prioritize one over another. It can only reason about the range within which each job's output length may fall.  

To handle this uncertainty, we introduce a naive benchmark algorithm, Max-Length Based Algorithm, denoted by \( \mathcal{A}_{\max} \). This algorithm assumes the worst-case scenario by treating every job as if its output length equals \( u \). At each time step, it selects the maximum number of prompts that can be processed without violating the memory constraint in Equation \eqref{eqn:Constraint}. Since the output lengths are overestimated, the algorithm guarantees that the memory limit is never exceeded, and thus no cancellations are required. The detail information can be found in Algorithm~\ref{algo:malb}.

\begin{algorithm}
\caption{$\mathcal{A}_{\max}$}
\label{algo:malb}
\KwIn{Memory Capacity $M$, prompt size $s$, output length lower and bound $\ell \leq o_i \leq u$ for all $i \in [n]$.}
\KwOut{Processing sequence $I = ( I_0, I_1, \cdots, I_T)$.}

\While{there exist waiting requests}{
    Let $R_t$ be the set of waiting prompts at time $t$. Let $S_t$ be the set of tokens currently processing at time $t$.
    
    \textbf{Take all $o_i = u$}, find the largest cardinality value $m_t$ such that there exists a set $U \subset R_t$ such that $|U|=m_t$ and Equation \eqref{eqn:Constraint} hold for all $t' \geq t$.
    
    Randomly sample $I_t \subset R_t$ with $|I_t| = m_t$.

    Process the requests in $I_t \cup S_t$ and update $R_{t+1} = R_t \backslash I_t$.
    }
\end{algorithm}

Although \( \mathcal{A}_{\max} \) guarantees feasibility, it may suffer from inefficiency due to its overly conservative estimation of output lengths. By assuming each job has the maximum possible length \( u \), it overestimates the peak memory usage per job and may significantly under-utilize the available KV cache memory. The following example illustrates this inefficiency:

\begin{example}
Consider a setting with \( n = 5 \) requests, each with input size \( s = 1 \), and output lengths \( o_i \in [1, 4] \). Suppose the true values are \( o_i = 1 \) for all \( i \in [5] \), and the memory capacity is \( M = 10 \). 

The hindsight-optimal algorithm \HS, having access to the true output lengths, knows that the peak memory usage per job is \( s + o_i = 2 \). Hence, it can batch all 5 jobs together, fully utilizing the memory \( M = 2 \times 5 = 10 \). Each job completes after 1 unit of processing time, resulting in a latency of 1 per job and total end-to-end latency of \( 5 \times 1 = 5 \).  

In contrast, \( \mathcal{A}_{\max} \) assumes \( o_i = 4 \) for all jobs, leading to a peak memory estimate of \( s + u = 5 \) per job. It therefore batches only \( \lfloor M / 5 \rfloor = 2 \) jobs at a time. It first processes 2 jobs, then the next 2, and finally the last remaining job. The resulting job completion times are 1, 1, 2, 2, and 3, yielding a total latency of \( 1 + 1 + 2 + 2 + 3 = 9 \).
\end{example}

This example highlights the inefficiency of \( \mathcal{A}_{\max} \), particularly in cases where the prediction interval is wide. As the gap between \( \ell \) and \( u \) increases—i.e., as \( \alpha = \ell/u \) decreases—the degree of memory under-utilization becomes more severe. The following theorem establishes an upper bound on the competitive ratio of \( \mathcal{A}_{\max} \).

\begin{theorem} \label{thm:Amax}
    The competitive ratio of  \( \mathcal{A}_{\max} \) is upper bounded by:
    \[
    \text{CR}(\mathcal{A}_{\max}) \leq \frac{\alpha^{-1}(1+\alpha^{-1})}{2}+\mathcal{O}(\frac{1}{M}).
    \]
\end{theorem}

While the example above provides intuition for where the inefficiency of \( \mathcal{A}_{\max} \) arises, formally analyzing its competitive ratio is extremely challenging. The difficulty stems from the complex dynamics of the model, particularly the fact that each job’s memory usage increases linearly over time during processing. Therefore, it is hard to find the closed-form connection between the total end-to-end latency and the memory constraint. To address this challenge, we introduces a combinatorial structure that captures the underlying memory evolution in a concise and analyzable way. This \emph{memory-preserving} structure serves as a key proof technique and offers a general framework for analyzing scheduling problems under similar memory constraints. We use a separate section to make the full proof, and one can find the proof in Appendix \ref{sec:memory}.

While Theorem \ref{thm:Amax} provides an upper bound of competitive ratio of  \( \mathcal{A}_{\max} \), this ratio is not tight because it treats all output lengths as equal to \(u\), which causes a mismatch between the effect of output lengths on latency and the effect of the input order. As a result, our analysis cannot fully capture how the actual input sequence influences the scheduling performance, leading to a loose upper bound. Then, we also provide a lower bound of competitive ratio of \( \mathcal{A}_{\max} \).

\begin{theorem} \label{thm:Amaxlower}
    The competitive ratio of \( \mathcal{A}_{\max} \) is lower bounded by 
    \[
    \text{CR}(\mathcal{A}_{\max}) \geq \frac{\alpha^{-1} (1+\alpha^{-1/2})}{2}.
    \]
\end{theorem}

\begin{proof}{Proof of Theorem \ref{thm:Amaxlower}}
    Take \(s=0\), recall \( \frac{\ell}{u} = \alpha \). Let \(M = \ell u\), and the algorithm \(\mathcal{A}_{\max}\) is assigned with \(N_\ell = a u\) prompts with output length \(\ell\) and \(N_u = b \ell\) prompts with output length \(u\). Here, we take \(\frac{a}{b} \approx \alpha^{1/2}\) and \(a, b\) are large enough. Because \(\mathcal{A}_{\max}\) regards all prompts as being of size \(u\), it will only initiate a new request if there is memory \(u\) available. This is equivalent to considering it as a combination of \(\frac{M}{u} = \ell\) parallel workers, where any one of them can only process one prompt at a time and start a new one randomly after finishing a prompt.

    Then it is simple to calculate \(\mathbb{E}[\text{TEL}(\mathbf{o};\mathcal A_{\max})] = (au + b\ell) \frac{au\ell + b\ell u}{2 \ell}\), since all prompts have the expected output time \(\frac{au\ell + b\ell u}{2 \ell}\). Also, \(\text{TEL}(\mathbf{o};\HS) = u\ell (1+\cdots+a) + b\ell a\ell + u \ell (1+\cdots+b)\). 

    We get 
    \[\text{CR}(\mathcal{A}_{\max}) \geq \frac{u}{l} \frac{(au + b\ell)(a+b)}{ua^2 + \ell2ab + ub^2} = \frac{\alpha^{-1} (1+\alpha^{-1/2})}{2}.\]
    \Halmos
\end{proof}

\section{Robust Advanced Scheduling Algorithm} \label{sec:algrobust}

Previously, we introduced a naive scheduling algorithm, \( \mathcal{A}_{\max} \), and analyzed its competitive ratio, which is asymptotically between \( \frac{\alpha^{-1} (1+\alpha^{-1/2})}{2}\) and \( \frac{\alpha^{-1}(1 + \alpha^{-1})}{2} \) according to Theorems \ref{thm:Amax} and \ref{thm:Amaxlower}, where \( \alpha = \frac{\ell}{u} \) captures the prediction accuracy of the output lengths \( o_i \). This result shows that the performance of \( \mathcal{A}_{\max} \) is highly sensitive to the quality of the prediction model. In the ideal case where the prediction is perfect (i.e., \( \ell = u \), so \( \alpha = 1 \)), the competitive ratio equals to 1, indicating optimal performance. However, in practice, decision-makers often prioritize speed over accuracy in prediction. When the prediction interval is wide (i.e., \( \alpha \to 0 \)), the competitive ratio grows unbounded
with the rate between $\mathcal{O}(\alpha^{-2})$ and $\mathcal{O}(\alpha^{-1.5})$, demonstrating that \( \mathcal{A}_{\max} \) is highly non-robust. This observation raises a key question: can we design a scheduling algorithm that performs well not only under accurate predictions, but also maintains a better competitive ratio under poor prediction quality? In this section, we introduce an advanced and robust policy denoted \( \mathcal{A}_{\min} \).

Recall that in \( \mathcal{A}_{\max} \), each request is pessimistically assumed to have the maximum output length \( u \). When the true output length \( o_i \) is much smaller than \( u \), this conservative estimate leads to significant memory overprotection, reducing the number of simultaneously processed requests and increasing latency. To address this, we approach the problem from the opposite perspective. If the goal is to maximize the number of requests in each batch, it is natural to use the lower bound \( \ell \) as a proxy for each request's output length. Unlike the upper bound \( u \), the lower bound can be dynamically updated: initially, we know only that \( o_i \geq \ell \), we create a variable $\tilde o_i$ for each request $i$ as its output length lower bound and initialize $\tilde o_i = \ell$ at beginning. But once a request has generated \( \tilde{\ell}_i \) tokens, we know \( o_i \geq \tilde{\ell}_i \), and we can update the lower bound accordingly, namely $\tilde o_i =\tilde{\ell}_i$.
Based on this idea, our advanced algorithm \( \mathcal{A}_{\min} \) proceeds in two key steps:
\begin{enumerate}
\item \textit{Memory Overflow Resolution via Ordered Eviction.  }
Because the true output length \( o_i \) is not known in advance and may exceed \( \tilde{o}_i \), the algorithm may encounter a situation where continuing to process the current batch would exceed the memory limit. In this case, \( \mathcal{A}_{\min} \) removes jobs from the batch to bring the memory usage back within the constraint. Specifically, it orders the currently active jobs by increasing \( \tilde{o}_i \), breaks ties randomly, and removes jobs one by one in this order until the projected memory usage becomes feasible. This strategy ensures that the algorithm prefers to retain requests with larger accumulated lower bounds, which are more expensive to restart. Moreover, for each canceled request \( i \), we update the value of \( \tilde{o}_i \)  to reflect the number of output tokens already generated, which is a lower bound of the true output length \( o_i \).

\item \textit{Batch Formation Based on Sorted Lower Bounds.  }
At time \( t = 0 \), the algorithm sets \( \tilde{o}_i = \ell \) for all requests \( i \). To form a batch, it sorts the current set of unprocessed or restartable requests in ascending order of \( \tilde{o}_i \), breaking ties uniformly at random. It then greedily selects as many requests as possible from the sorted list—starting from the smallest \( \tilde{o}_i \)—until the memory constraint is reached. 

\end{enumerate}

Importantly, since the sampling process is greedy and deterministic, removing a request \( i \) does not change its true output length \( o_i \). Thus, even if a request \( i \) is removed and later restarted, the algorithm retains its current lower bound \( \tilde{o}_i \) at the time of removal. This updated value is used for future batch formation. The auxiliary lower bound $\tilde{o}_i$ is updated only when a job $i$ is cancelled, since a job continues processing uninterrupted once selected into a batch; thus, there is no need to update $\tilde{o}_i$ while the job remains active. Upon cancellation, however, the number of output tokens already generated provides a new certified lower bound on the true output length $o_i$, and this updated $\tilde{o}_i$ is used to inform future batch decisions. This strategy avoids unnecessary computation while preserving the algorithm’s adaptiveness and memory efficiency. 

\textit{Key Design Advantage: No Need for Upper Bound Prediction. }  An important feature of $\mathcal{A}_{\min}$ is that it operates solely using the predicted lower bounds $\ell$, without ever relying on the predicted upper bounds $u$. This makes the algorithm significantly more practical and robust: in many real-world settings, generating reliable lower bounds is much easier and faster than estimating sharp upper bounds.  By avoiding dependence on $u$, $\mathcal{A}_{\min}$ remains effective even when the prediction intervals are highly uncertain or asymmetric, making it more suitable for deployment in systems where inference-time estimates must be generated quickly or under weak supervision. The formal pseudocode for \( \mathcal{A}_{\min} \) is presented in Algorithm~\ref{algo:milb}.

\begin{algorithm}
\caption{$\mathcal{A}_{\min}$}   
\label{algo:milb}
\KwIn{Memory capacity $M$, prompt size $s$, output length lower bounds $o_i \geq \ell$ for all $i \in [n]$.}
\KwOut{Processing sequence $I = ( I_0, I_1, \cdots, I_T )$.}

Initialize $\tilde{o}_i \gets \ell$ for all $i \in [n]$.

\While{there exist unfinished requests}{
    Let $R_t$ be the set of waiting prompts at time $t$. Let $S_t$ be the set of tokens currently processing at time $t$.

    \If{projected memory usage of $S_t$ at time $t+1$ exceeds $M$}{
        Sort $S_t$ in ascending order of $\tilde{o}_i$; break ties uniformly at random.

        Remove jobs one by one from $S_t$ (following the sorted order) until projected memory usage at $t+1$ satisfies the memory constraint.

        For each removed request $i$, update $\tilde{o}_i \gets$ number of tokens already generated by request $i$.
    }

    Let $S_t'$ be the set of remaining active requests after any removals.

    Sort $R_t$ in ascending order of $\tilde{o}_i$; break ties uniformly at random.

    Select the largest subset $I_t \subset R_t$ (in order) such that adding $I_t$ to $S_t'$ satisfies the memory constraint in Equation~\eqref{eqn:Constraint}.

    Process the requests in $I_t \cup S_t'$.

    Update $R_{t+1} = R_t \setminus I_t$.
}
\end{algorithm}

Next, we start to analyze the theoretical performance of $\mathcal A_{\min}$. First, the following theorem states the computational complexity of $\mathcal A_{\min}$ is polynomial, and the proof can be found in Appendix \ref{append:algrobust}. 

\begin{theorem}
\label{thm:complexity}
    Consider the KV cache memory limit is $M$, Algorithm $\mathcal A_{\min}$ has  a computational complexity of $\mathcal{O}(M log M)$.
\end{theorem}

Next, Theorem \ref{thm:optimality} states that no algorithm can achieve a better performance than $\mathcal{A}_{\min}$ when $M \to \infty$, showing the asymptotically optimality. The proof can also be found in Appendix \ref{append:algrobust}.

\begin{theorem} \label{thm:optimality}
    For any feasible policy $\pi$ without the full information of $\mathbf{o}=(o_1,o_2,\ldots,o_n)$, when $M \to \infty$, the competitive ratio is lower bounded by
    \[
    \text{CR}(\pi) \geq \text{CR}(\mathcal{A}_{\min}).
    \]
\end{theorem}

Although Theorem \ref{thm:optimality} proves the optimality of $\mathcal{A}_{\min}$, it is still necessary to analyze its competitive ratio to quantify its improvement over $\mathcal{A}_{\max}$. However, like $\mathcal{A}_{\max}$, deriving an upper bound on this ratio is highly challenging. The next subsection presents our proof approach and theoretical results.

\subsection{Logarithm Competitive Ratio Bound of $\mathcal{A}_{\min}$} \label{subsec:log}

In this subsection, our main purpose is to introduce our novel method to prove the following theorem:

\begin{theorem} \label{thm:aminmain}
    While $M \to \infty$, the asymptotical competitive ratio of \( \mathcal{A}_{\min} \) is $\mathcal{O}\left(\log (\alpha^{-1}) \right)$.
\end{theorem}

To prove Theorem \ref{thm:aminmain}, we divide the analysis into two parts. In Part 1, we derive a closed-form expression for the competitive ratio of \( \mathcal{A}_{\min} \). In Part 2, we rigorously bound this expression and show that it scales logarithmically.

\textbf{Part 1: Get a Closed-form expression for the competitive ratio.}

We begin with two key observations:
\begin{itemize}
    \item \textit{Observation 1: } The execution of $\mathcal{A}_{\min}$ can be partitioned into several periods: $\tau_{\ell}, \tau_{\ell+1}, \cdots, \tau_{u}$. Before the start of the period $\tau_i$, all requests $q$ with real output length $o_q < i$ have already been completed. In period $\tau_i$, $\mathcal{A}_{\min}$ attempts to initialize all requests with a lower bound on output length $\Tilde{o} = i$. For example, in the first period, all requests are loaded into memory and processed until they are either completed or deleted. If a request $q$ is deleted, its associated lower bound is updated to a new value $\Tilde{o}_q > i$. During $\tau_i$, we denote by $b_{ij}$, $c_{ij}$, and $d_{ij}$ the numbers of input, completed, and deleted requests with the true output length $o = j$, respectively. These quantities satisfy the following relationship.
    \begin{lemma}
    \label{lem:deletion}
    We have
    \[
    \mathbb{E} [c_{ij} | b_{ij}] = \frac{s+i}{s+j} b_{ij}, \quad \mathbb{E} [d_{ij} | b_{ij}] = \left(1 - \frac{s+i}{s+j}\right) b_{ij}.
    \]
    \end{lemma}
    The proof of Lemma \ref{lem:deletion} can be found in Appendix \ref{append:algrobust}.

\item \textit{Observation 2: } In period $\tau_i$, if a request $q$ is deleted and reassigned a new lower bound $\Tilde{o}_q = k$, its true output length must satisfy $o_q \geq \Tilde{o}$. Again, by input symmetry, the probability that its true output length is $o = j \ge k$ is proportional to $b_{ij}$. This is because $\mathcal{A}_{\min}$ breaks ties uniformly when selecting inputs. Formally, we have 
\begin{equation}\label{eq:update_distribution}
    \mathbb{P}_i (o = j \mid \Tilde{o} = k) = \frac{b_{ij}}{\sum_{t \ge k} b_{it}}. 
\end{equation}
\end{itemize}

For a fixed request set $\mathbf{o}_n$ of $n$ requests, where $x_i$ is the proportion of the number of requests with output length $i$ in the request set $\mathbf{o}_n$, the above two observations allow us to derive the expected state transitions across all periods. 

\begin{proposition}
\label{prop:state_transition}
Given the request set $\mathbf{o}_n$, we have
\begin{align}
\mathbb{E} [c_{ij}] = 
\begin{cases}
\displaystyle \frac{s+i}{s+j} x_j n , & \text{if } i = \ell, \\
\displaystyle \frac{1}{s+j} x_j n , & \text{otherwise.}
\end{cases}
\end{align}
\end{proposition}

The proof of Proposition \ref{prop:state_transition} can be found in Appendix \ref{append:algrobust}. Proposition \ref{prop:state_transition} indicates that in the initial period $\tau_{\ell}$, $\mathcal{A}_{\min}$ successfully processes more requests with shorter lengths. Moreover, the proportion of completed requests decreases with $j$ according to the factor $\frac{s+\ell}{s+j}$. In the subsequent periods, the expected number of completed requests remains constant at $\frac{1}{s+j}$ times $x_j n$. Although this expected completion ratio does not change across periods, $\mathcal{A}_{\min}$ achieves it by continually updating the lower bound $\Tilde{o}$ in earlier periods. This demonstrates that $\mathcal{A}_{\min}$ can effectively defer the input of longer requests by actively learning and utilizing structural information from previous steps. We can now derive the competitive ratio of $\mathcal{A}_{\min}$.

\begin{lemma}
\label{thm:Amin}
    Let 
    \[
    \delta_{s,l}(k) =
    \begin{cases}
    s + l, & \text{if } k = l \\
    1, & \text{if } k \ne l
    \end{cases}
    \]
    We have
    \[
    \text{CR}(\mathcal{A}_{\min}) = \frac{\sum_{k=\ell}^{u} \delta_{s,l}(k) (\sum_{i=k}^{u} i x_i) (\sum_{i=k}^{u} \frac{\delta_{s,l}(k)}{s+i}x_i + 2 \sum_{i=k+1}^{u} \frac{i-k}{s+i} x_i)}{\sum_{k=\ell}^{u} k(s+k) x_k (x_k + 2 \sum_{i=k+1}^{u} x_i)} + \mathcal{O}(\frac{1}{M}).
    \]
\end{lemma}

\proof{Proof of Lemma \ref{thm:Amin}}
We compute the two components of
\[
\frac{\mathbb{E}[\text{TEL}(\mathbf{o}_n;\mathcal{A}_{\min})]}{\text{TEL}(\mathbf{o}_n;\HS)}.
\]
For the hindsight algorithm, which processes requests starting from the shortest ones, the total processing time for requests of output length $k$ is $\frac{k(s+k) x_k n }{M}$. At this point, requests with length $k$ are symmetric, and the average latency increase for these requests is $\frac{k(s+k) x_k n}{2 M}$. Additionally, the remaining $\sum_{i=k+1}^{u} x_i n$ requests must wait, incurring a further latency of $\frac{k(s+k) x_k n }{M}$.

Summing over all $k$ from $\ell$ to $u$, we obtain
\[
\text{TEL}(\mathbf{o}_n;\HS) = \frac{n^2}{2M} \left( \sum_{k=\ell}^{u} k(s+k) x_k \left(x_k + 2 \sum_{i=k+1}^{u} x_i\right) \right) (1+\mathcal{O}(\tfrac{1}{M})).
\]

For the algorithm $\mathcal{A}_{\min}$, it successfully completes a portion of the requests in each stage, leading to an expected time cost of
\[
\mathbb{E}[\tau_k] = \sum_{i=k}^{u} \frac{i(s+i)}{M} \mathbb{E}[c_{ii}] = \frac{\delta_{s,l}(k)}{M} \sum_{i=k}^{u} i x_i n.
\]
Furthermore, the expected number of remaining requests is $\sum_{i=k+1}^{u} \frac{i-k}{s+i} x_i n$.

Summing over $k$ from $\ell$ to $u$, we get
\begin{align*}
    \mathbb{E}[\text{TEL}(\mathbf{o}_n;\mathcal{A}_{\min})] =& \frac{n^2}{2M} \left( \sum_{k=\ell}^{u} \delta_{s,l}(k) \left(\sum_{i=k}^{u} i x_i\right) \left(\sum_{i=k}^{u} \frac{\delta_{s,l}(k)}{s+i}x_i + 2 \sum_{i=k+1}^{u} \frac{i-k}{s+i} x_i\right) \right) (1+\mathcal{O}(\tfrac{1}{M})).
\end{align*}

Dividing by $\text{TEL}(\mathbf{o}_n;\HS)$ and taking the limit as $n \to \infty$, we conclude that
\[
\text{CR}(\mathcal{A}_{\min}) = \frac{\sum_{k=\ell}^{u} \delta_{s,l}(k) (\sum_{i=k}^{u} i x_i) (\sum_{i=k}^{u} \frac{\delta_{s,l}(k)}{s+i}x_i + 2 \sum_{i=k+1}^{u} \frac{i-k}{s+i} x_i)}{\sum_{k=\ell}^{u} k(s+k) x_k (x_k + 2 \sum_{i=k+1}^{u} x_i)} + \mathcal{O}(\frac{1}{M}).
\]
\Halmos
\endproof

To see the asymptotic performance of $\text{CR}(\mathcal{A}_{\min})$, W.L.O.G., we take $s=0$, $\ell=1$, then, $\alpha = \frac{\ell}{u}=\frac{1}{u}$. The next corollary writes the competitive ratio of $\mathcal{A}_{\min}$ in the matrix form: \begin{corollary}
\label{cor:A/B}
    Let $\vec{x} = (x_1, \cdots, x_u)^\top$, and define $\mathbf{A}_u = (a_{ij})_{u \times u}$ and $\mathbf{B}_u = (b_{ij})_{u \times u}$, where
    \[
    a_{ij} = \frac{\min(i,j)^2}{ij} \left( \frac{1}{2}(i+j)^2 - \min(i,j)^2 \right),
    \]
    \[
    b_{ij} = \min(i,j)^2.
    \]
    The competitive ratio of $\mathcal{A}_{\min}$ simplifies to the following Rayleigh quotient,
    \[
    \text{CR}(\mathcal{A}_{\min}) = \max_{\vec{x}}\frac{\vec{x}^\top \mathbf{A}_u \vec{x} }{\vec{x}^\top \mathbf{B}_u \vec{x}} + \mathcal{O}(\tfrac{1}{M}).
    \]
\end{corollary}

The proof of Corollary \ref{cor:A/B} can be found in Appendix \ref{append:algrobust}. Our goal is to show that the Rayleigh quotient
\[
R(\vec{x}) := \frac{\vec{x}^\top \mathbf{A}_u \vec{x}}{\vec{x}^\top \mathbf{B}_u \vec{x}}
\]
is uniformly bounded by $\mathcal{O}(\log u)$ for all $\vec{x} > 0$.

\textbf{Part 2: Bound the Rayleigh quotient.}

To provide the upper bound of $R(\vec{x})$, we start by showing that the two matrices $\mathbf{A}_u$ and $\mathbf{B}_u$ are both positive definite. 

\begin{lemma}[Positive Definiteness of $\mathbf{A}_u$ and $\mathbf{B}_u$]
\label{lem:positive-definite}
For all integers $u \geq 1$, both $\mathbf{A}_u$ and $\mathbf{B}_u$ are positive definite.
\end{lemma}

The proof of Lemma \ref{lem:positive-definite} can be found in Appendix \ref{append:algrobust}.

Consider the matrix
\[
\mathbf{C}_u := \mathbf{B}_u^{-1} \mathbf{A}_u.
\]
Since both \( \mathbf{A}_u \) and \( \mathbf{B}_u \) are symmetric and positive definite (by Lemma~\ref{lem:positive-definite}), the matrix \( \mathbf{C}_u \) is similar to a symmetric positive definite matrix, and thus has real, strictly positive eigenvalues. In particular, the maximum of the Rayleigh quotient satisfies
\[
\max_{\vec{x} > 0} R(\vec{x}) \le \max_{\vec{x} \in \mathbb{R}^u} R(\vec{x}) = \rho(\mathbf{C}_u),
\]
where \( \rho(\mathbf{C}_u) \) denotes the spectral radius, i.e., the largest eigenvalue of \( \mathbf{C}_u \).

Since all eigenvalues of \( \mathbf{C}_u \) are positive, its spectral radius is bounded above by the sum of its eigenvalues. Therefore, it suffices to estimate the trace:
\[
\rho(\mathbf{C}_u) \le \mathrm{tr}(\mathbf{C}_u) = \mathrm{tr}(\mathbf{B}_u^{-1} \mathbf{A}_u).
\]

We now compute the trace of \( \mathbf{C}_u = \mathbf{B}_u^{-1} \mathbf{A}_u \) explicitly using the known expressions for \( \mathbf{A}_u \) and the inverse of \( \mathbf{B}_u \). Recall that \( \mathbf{B}_u^{-1} \) is tridiagonal, with entries
\[
(\mathbf{B}_u^{-1})_{ij} =
\begin{cases}
\displaystyle \frac{4i}{4i^2 - 1}, & i = j < u, \\
\displaystyle \frac{1}{2u - 1}, & i = j = u, \\
\displaystyle -\frac{1}{2\min(i,j) + 1}, & |i - j| = 1, \\
0, & \text{otherwise}.
\end{cases}
\]

Using the symmetry of \( \mathbf{A}_u \), the trace can be written as the sum of products over all entries of \( \mathbf{B}_u^{-1} \) and \( \mathbf{A}_u \) along the main and immediate subdiagonals:
\[
\mathrm{tr}(\mathbf{B}_u^{-1} \mathbf{A}_u)
= \sum_{i=1}^u (\mathbf{B}_u^{-1})_{ii} a_{ii} + 2 \sum_{i=1}^{u-1} (\mathbf{B}_u^{-1})_{i,i+1} a_{i,i+1}.
\]
We define the two contributions as:
\begin{align*}
\text{(diagonal terms)}, \quad T_1(u) &:= \sum_{i=1}^u (\mathbf{B}_u^{-1})_{ii} a_{ii} = \sum_{i=1}^{u-1}\frac{4i}{4i^2-1}i^2 + \frac{1}{2u-1} u^2, \\
 \text{(subdiagonal terms)},  \quad T_2(u) &:= \sum_{i=1}^{u-1} (\mathbf{B}_u^{-1})_{i,i+1} a_{i,i+1} = - \sum_{i=1}^{u-1} \frac{1}{2i+1} \frac{i}{i+1} (i^2 +2i + \frac{1}{2}),
\end{align*}
so that
\[
\mathrm{tr}(\mathbf{B}_u^{-1} \mathbf{A}_u) = T_1(u) + 2 T_2(u).
\]

Explicit computation yields the following closed-form expressions:

- For the diagonal term:
\[
T_1(u) = \frac{u^2 + 1}{2} + \frac{1}{2} \sum_{i=1}^{u-1} \frac{1}{2i+1}.
\]

- For the subdiagonal term:
\[
2 T_2(u) = - \frac{u^2-1}{2} + \sum_{i=1}^{u-1} \frac{1}{i+1} - \frac{1}{2} \sum_{i=1}^{u-1} \frac{1}{2i+1}.
\]

Combining all terms:
\[
\mathrm{tr}(\mathbf{B}_u^{-1} \mathbf{A}_u)= T_1(u) + 2 T_2(u) = \sum_{i=1}^{u} \frac{1}{i} = \mathcal{O}(\log u).
\]

Therefore, the spectral radius of \( \mathbf{C}_u \) is also \( \mathcal{O}(\log u) \), and the Rayleigh quotient satisfies:
\[
\max_{\vec{x} > 0} \frac{\vec{x}^\top \mathbf{A}_u \vec{x}}{\vec{x}^\top \mathbf{B}_u \vec{x}} = \mathcal{O}(\log u),
\]
which implies that the competitive ratio of the algorithm \( \mathcal{A}_{\min} \) satisfies:
\[
\text{CR}(\mathcal{A}_{\min} ; \mathcal{D}) = \mathcal{O}\left( \log u \right) = \mathcal{O}\left(\log (\alpha^{-1})\right).
\]
\Halmos 

We now have analyzed the Rayleigh quotient theoretically, establishing a logarithmic competitive ratio for $\mathcal{A}_{\min}$. Finally, we numerically estimate the Rayleigh quotient for the case $s=0$ and $\ell=1$ to verify the theoretical result. As shown in Figure \ref{fig:log_CR}, the logarithmic function provides an excellent fit, with an $R^2$ value exceeding $0.9999$.

\begin{figure}[!h]
\center
\includegraphics[width=0.6\textwidth]{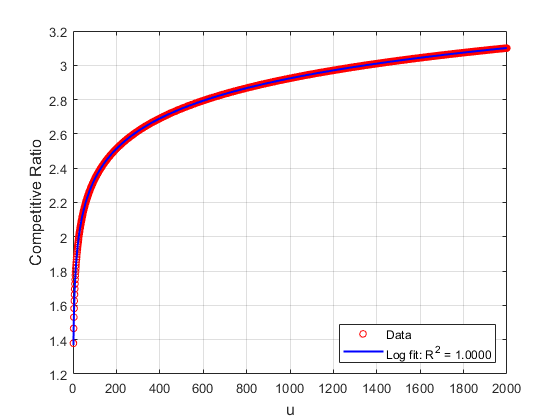}
\caption{Competitive ratio of $\text{A}_{\min}$ versus upper bound $u$, fitted by $0.2555 \log(u + 0.2793) + 1.1587$, assuming $s = 0$, $l = 1$.} 
\label{fig:log_CR}
\end{figure}

\subsection{Heterogeneous Prediction Intervals} \label{subsec:hetero}

Now, we extend our framework by replacing the single prediction interval with $m$ disjoint intervals: ${[\ell_1,u_1], [\ell_2,u_2], \ldots, [\ell_{m},u_m]}$, satisfying $u_j \le \ell_{j+1}$. This formulation naturally models classification-based prediction approaches, where a machine learning model first predicts an output length for each request and then assigns it to the corresponding interval. 

Our algorithm, \( \mathcal{A}_{\min} \), naturally extends to this generalized setting with minimal modifications. For each request $i$ classified into interval $[\ell_j, u_{j}]$, we initialize its lower bound as $\ell_j$ rather than using a uniform initial value across all requests. This modified version (Algorithm \ref{algo:milbrevise} in Appendix \ref{append:algrobust}) guarantees that the initial underestimation becomes more precise for each individual request, and the algorithm consequently achieves better overall performance. As formalized in the following corollary, this adaptation yields a competitive ratio at least equals to the basic setting. We validate the numerical performance in Experiment 2 in Section \ref{sec:num}, demonstrating the practical benefits of incorporating refined prediction intervals.

\begin{corollary}
    Under prediction intervals ${[\ell_1,u_1], [\ell_2,u_2], \ldots, [\ell_{m},u_m]}$, satisfying $u_j \le \ell_{j+1}$, take $\alpha= \frac{\ell_1}{u_m}$. While $M \to \infty$, the asymptotical competitive ratio of the revised \( \mathcal{A}_{\min} \) is $\mathcal{O}\left(\log (\alpha^{-1}) \right)$.
\end{corollary}

\section{Extensions: $\mathcal{A}_{\min}$ and New Algorithms under Specific Output Distributions} \label{sec:extension}

In previous sections, we assumed the adversary selects request output values within the prediction interval $[\ell, u]$. Here, we examine another variant where outputs follow some specific distributions $\mathcal{D}$ over $[\ell, u]$, and analyze $\mathcal{A}_{\min}$'s performance under this setting. Notably, $\mathcal{A}_{\min}$ operates solely using $\ell$ without distributional knowledge. This motivates our investigation of new algorithms that can leverage $\mathcal{D}$ to improve performance.

We focus on three representative distributions:
\begin{enumerate}
\item Two-point distribution: All outputs are either $\ell$ or $u$, representing an extreme case;
\item Geometric distribution: Outputs exhibit exponential decay, a pattern commonly observed in real-world datasets.
\item Linearly weighted Geometric distribution: Combines exponential decay with a linearly weighted preference for values near the lower bound $\ell$. This modified distribution better captures real-world phenomena where extreme values occur more frequently than standard geometric decay would predict.
\end{enumerate}

\subsection{Two-Point Distribution}

In this section, we examine a special case where the output length $o \in \{\ell, u\}$. We denote this two-point distribution family as $\mathcal{D}_2$. We starts by evaluating the theoretical performance of $\mathcal A_{\min}$. 

\begin{theorem} \label{thm:no}
    Under any distribution $\mathcal D \in \mathcal{D}_2$, the competitive ratio of $\mathcal{A}_{\min}$ is bounded by:
    \[
    \text{CR}(\mathcal{A}_{\min}) \leq \frac{3-\alpha}{2}+\mathcal{O}(\frac{1}{M}).
    \]
\end{theorem}

The complete proof appears in Appendix \ref{append:extension}. Theorem \ref{thm:no} establishes that for any $\alpha \in [0,1]$, the competitive ratio of $\mathcal{A}_{\min}$ is bounded above by $\frac{3}{2}$. This represents a significant improvement over $\mathcal{A}_{\max}$, whose competitive ratio is unbounded in the worst case – a scenario that occurs precisely within the two-point distribution family $\mathcal{D}_2$. Our result thus demonstrates that $\mathcal{A}_{\min}$ achieves bounded competitiveness where $\mathcal{A}_{\max}$ fails.

Next, we introduce the \emph{promote--$\ell$ policy} $\mathcal{A}_\ell$, which uses the information that the output distribution belongs to $\mathcal{D}_2$. Under this policy, the server processes each incoming request sequentially. As soon as the currently served request generates $\ell$ tokens, the policy infers that the true length of this request must be $u>\ell$. At that point, $\mathcal{A}_\ell$ removes the request, appends it to the end of the queue, and proceeds to the next one. Any request of length~$\le \ell$ is completed normally. In effect, $\mathcal{A}_\ell$ defers all long jobs after $\ell$ tokens, postponing them until all shorter jobs are finished. The detailed algorithm can be found in Algorithm \ref{algo:promote_l} in Appendix \ref{append:extension}. Next, we provide the theoretical competitive analysis of $\mathcal{A}_\ell$. 

\begin{theorem}\label{thm:Al_CR}
Under any distribution $\mathcal D \in \mathcal{D}_2$, the competitive ratio of $\mathcal{A}_\ell$ is bounded by:
\[
\textnormal{CR}(\mathcal{A}_\ell)
\;\le\;
\begin{cases}
\displaystyle 1+\dfrac{\alpha}{2(1-\alpha)}, & 0<\alpha\le\frac12,\\[10pt]
1+2\alpha^{2}, & \dfrac12\le\alpha\le1,
\end{cases}
\;+\;\mathcal{O}\!\bigl(\tfrac1M\bigr),
\qquad\text{where }\alpha=\ell/u.
\]
\end{theorem}

The proof of Theorem \ref{thm:Al_CR} can be found in Appendix \ref{append:extension}. Next, based on the results in Theorems \ref{thm:no} and \ref{thm:Al_CR}, we can study when to choose $\mathcal{A}_{\min}$ and $\mathcal{A}_\ell$.
From Theorem~\ref{thm:no}, we know that
\(
\text{CR}(\mathcal{A}_{\min})\le(3-\alpha)/2+\mathcal{O}(1/M)
\),
while Theorem~\ref{thm:Al_CR} gives a piecewise upper bound for
\(\text{CR}(\mathcal{A}_\ell)\).
To compare the two when \(0\le\alpha\le\frac12\), we consider:
\[
\text{CR}(\mathcal{A}_\ell)\;<\;\text{CR}(\mathcal{A}_{\min})
\quad\Longleftrightarrow\quad
1+\frac{\alpha}{2(1-\alpha)}
\;<\;\frac{3-\alpha}{2},
\]
which reduces to the quadratic inequality \(\alpha^2-3\alpha+1=0\).  
Solving yields the critical threshold:
\[
\alpha^\star=\frac{3-\sqrt5}{2}\approx0.382.
\]
Since \(\alpha=\ell/u\), the condition \(\alpha<\alpha^\star\) translates to
\(u\gtrsim2.62\,\ell\); that is, \emph{when the long request is more than 2.6 times larger than the short one, $\mathcal{A}_\ell$ achieves a smaller competitive ratio.}  
For \(\alpha\ge\alpha^\star\), the monotonic bound \(1+2\alpha^2\) dominates, making $\mathcal{A}_{\min}$ the preferable choice.

\begin{figure}[h]
  \centering
  \includegraphics[width=.55\linewidth]{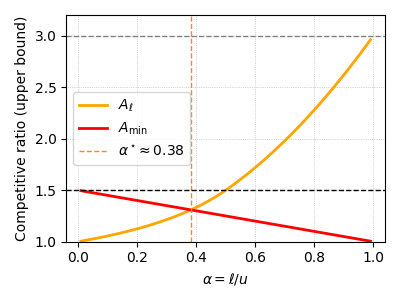}%
  \vspace{-3pt}
  \caption{Upper-bound curves for the competitive ratios under
  \(\mathcal{D}_2\).  When \(\alpha<\alpha^\star\approx0.38\)
  (left of the dashed line), 
  \(\mathcal{A}_\ell\) outperforms 
  \(\mathcal{A}_{\min}\); beyond that point, the inequality reverses.}
  \label{fig:CRcompare}
\end{figure}

\xhdr{Practical takeaway.}
Rather than committing to a single scheduling policy, one can \emph{adaptively} choose between $\mathcal{A}_{\min}$ and $\mathcal{A}_\ell$ based on the empirical ratio
\(\alpha=\ell/u\) observed in the workload. This simple switching strategy ensures the smaller of the two analytical bounds shown in Figure~\ref{fig:CRcompare}, improving the worst-case guarantee
from \(1.5\) to~\(1+\tfrac{\alpha}{2(1-\alpha)}\) in high-skew scenarios, while never exceeding the red curve in low-skew regimes. More broadly, this illustrates a general principle: taking advantage of distributional features, such as the length ratio~\(\alpha\), to select among multiple algorithms, can significantly enhance the robustness of competitive-ratio guarantees.

\subsection{Geometric Distribution}

Let $G(p)$ denote the geometric distribution with parameter $p$, such that $o \sim G(p)$ and $\mathbb{P}(o = k) = p q^{k-1}$, where $q = 1 - p$. The $G(p)$ describes an exponential decay of $o$ on the discrete space $[1,u]$. Under the geometric distribution, it is hard to provide a theoretical result for the competitive ratio. However, the left plot in Figure \ref{fig:G&LG} indicates that the competitive ratio of $\mathcal A_{\min}$ increases when $q$ increases, and is bounded by $1.7$. 

\begin{figure}[!h]
\centering
\begin{subfigure}{0.45\textwidth}
    \includegraphics[width=\linewidth]{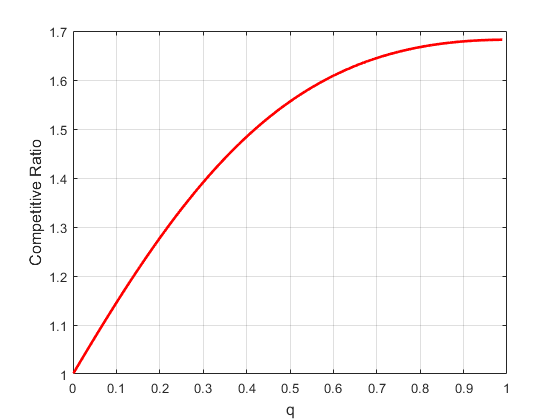}
    \label{fig:geometric}
\end{subfigure}
\hspace{0.05\textwidth}
\begin{subfigure}{0.45\textwidth}
    \includegraphics[width=\linewidth]{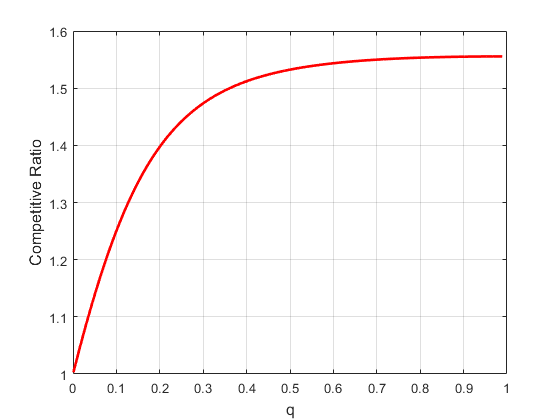}
    \label{fig:linear_geometric}
\end{subfigure}
\caption{Competitive ratio as a function of parameter $q$: (Left) geometric distribution $G(p)$; (Right) linearly weighted geometric distribution $LG(p)$.}
\label{fig:G&LG}
\end{figure}

\subsection{Linearly Weighted Geometric Distribution}

Let $LG(p)$ denote the linearly weighted geometric distribution with the parameter $p$, $o \sim LG(p)$ and $\mathbb{P}(o = k) = k p^2 q^{k-1}$, where $q = 1-p$. Compared to $G(p)$, the linearly weighted geometric distribution more accurately reflects practical scenarios, as it exhibits increasing mass for small values of $o$ and reaches its peak around $\log \tfrac{1}{q}$.

The right plot of Figure \ref{fig:G&LG} shows that the competitive ratio of $\mathcal A_{\min}$ again increases when $q$ increases, and is bounded by $1.6$. Furthermore, a closed-form expression can be derived for the linearly weighted geometric distribution $LG(p)$, and the tight competitive ratio upper bound is about $1.56$. The following theorem shows the result and the proof can be found in Appendix \ref{append:extension}.

\begin{theorem}
\label{cor:LG}
As $u \to \infty$ and $M \to \infty$, the competitive ratio of $\mathcal{A}_{\min}$ under any linearly weighted geometric distribution $\mathcal D \in LG(p)$ is given by
\[
\text{CR}(\mathcal{A}_{\min}) = \frac{(1+q)^2 (1 + 3q + 6q^2 + 3q^3 + q^4)}{1 + 2q + 11q^2 + 8q^3 + 11q^4 + 2q^5 + q^6} \le \frac{14}{9} \approx 1.56.
\]
\end{theorem}

\section{Numerical Experiments} \label{sec:num}

In this section, we evaluate the performance of different scheduling algorithms using a real-world dataset under various assumptions about the accuracy of output length predictions.

\xhdr{Dataset Overview.}
We use the LMSYS-Chat-1M dataset released by \cite{zheng2023lmsys}, available at \url{https://huggingface.co/datasets/lmsys/lmsys-chat-1m}, which consists of conversations collected from over 210,000 unique IP addresses via the Vicuna demo and the Chatbot Arena platform. For our numerical experiments, we randomly sample a subset of 2,000 conversations.

For each conversation, we define the input size as the number of tokens in the prompt and the output length as the number of tokens in the corresponding model response. Among the selected 2,000 conversations, the input sizes range from 1 to 468 tokens, with a mean of 41, median of 11, and variance of 4,961. The output lengths range from 1 to 883 tokens, with a mean of 85, median of 43, and variance of 9,702. Figure \ref{fig:distribution} displays the empirical distributions of input sizes and output lengths.

\begin{figure}[!ht]
\center
\includegraphics[width=0.8\textwidth]{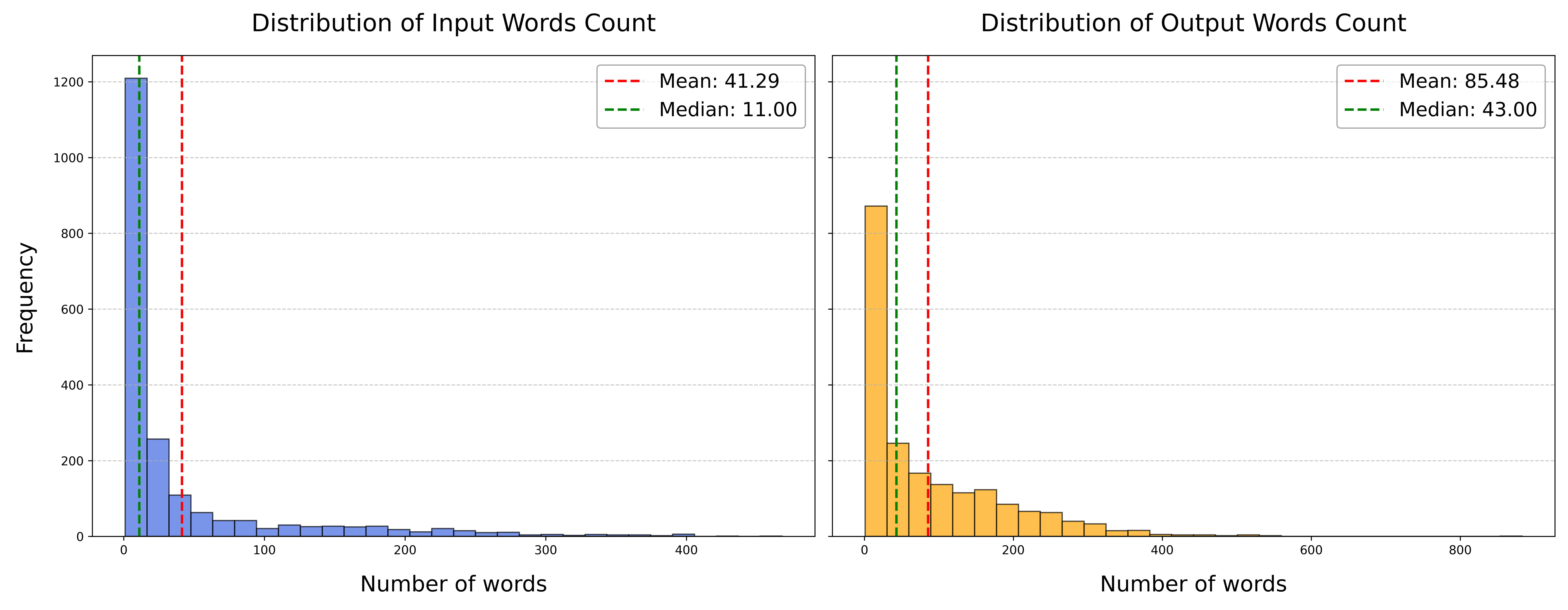}
\caption{Distribution of the number of words of input prompt and output response respectively} 
\label{fig:distribution}
\end{figure}

\xhdr{Experiment Setup.}
We evaluate the scheduling algorithms under three prediction settings that reflect varying levels of output length accuracy:
\begin{enumerate}
\item \textbf{Rough Prediction: } Each request is assumed to have a coarse prediction interval of $[1, 1000]$, representing minimal information about the output length.
\item \textbf{Non-Overlapping Classification: } Based on the true output length of each request, we assign its prediction interval to one of the fixed buckets: $[1, 100], [101, 200], \ldots, [901, 1000]$, simulating a multi-class classification model with non-overlapping intervals.
\item \textbf{Overlapping Interval Prediction: } For each request $i$ with true output length $o_i$, the prediction interval is set as $[ (1 - x) o_i, (1 + x) o_i ]$, where $x \in (0,1)$ controls the accuracy of the prediction. Larger values of $x$ correspond to more precise predictions.
\end{enumerate} 

To observe latency trends under different load levels, we vary the number of requests in each simulation, selecting random subsets of size 200, 400, ..., up to 2,000 from the original dataset. We report the average end-to-end latency, defined as the total latency divided by the number of requests.

In all experiments, we compare the average latency of $\mathcal{A}_{\max}$ and $\mathcal{A}_{\min}$. As a benchmark, we include the \HS{} algorithm, which assumes perfect knowledge of each request's output length and serves as a lower bound on achievable latency. Batch processing time is estimated using the linear regression predictor in Vidur simulator from \cite{agrawal2024vidur} under the simulation environment of the LLaMA2-70B model on two linked A100 GPUs.


\xhdr{Results for Experiment 1.}
In Experiment 1, the prediction interval for all requests is set to the broad range $[1, 1000]$. Under this setting, $\mathcal{A}_{\max}$ pessimistically assumes that each request has an output length of 1000. While this conservative approach avoids memory overflows, it leads to poor memory utilization—each batch includes very few requests, resulting in high latency. In contrast, $\mathcal{A}_{\min}$ initializes the lower bound of each request as 1 and adaptively updates this value during the inference process as tokens are generated.

Figure~\ref{fig:exp1} shows that $\mathcal{A}_{\min}$ achieves average latency nearly identical to the benchmark \HS, which has full knowledge of the true output lengths. This result is particularly striking given that $\mathcal{A}_{\min}$ operates with minimal information—only that each request's output lies somewhere in the interval $[1, 1000]$. It highlights the robustness and adaptiveness of $\mathcal{A}_{\min}$ even under highly uncertain predictions.

\begin{figure}[!ht]
\center
\includegraphics[width=0.8\textwidth]{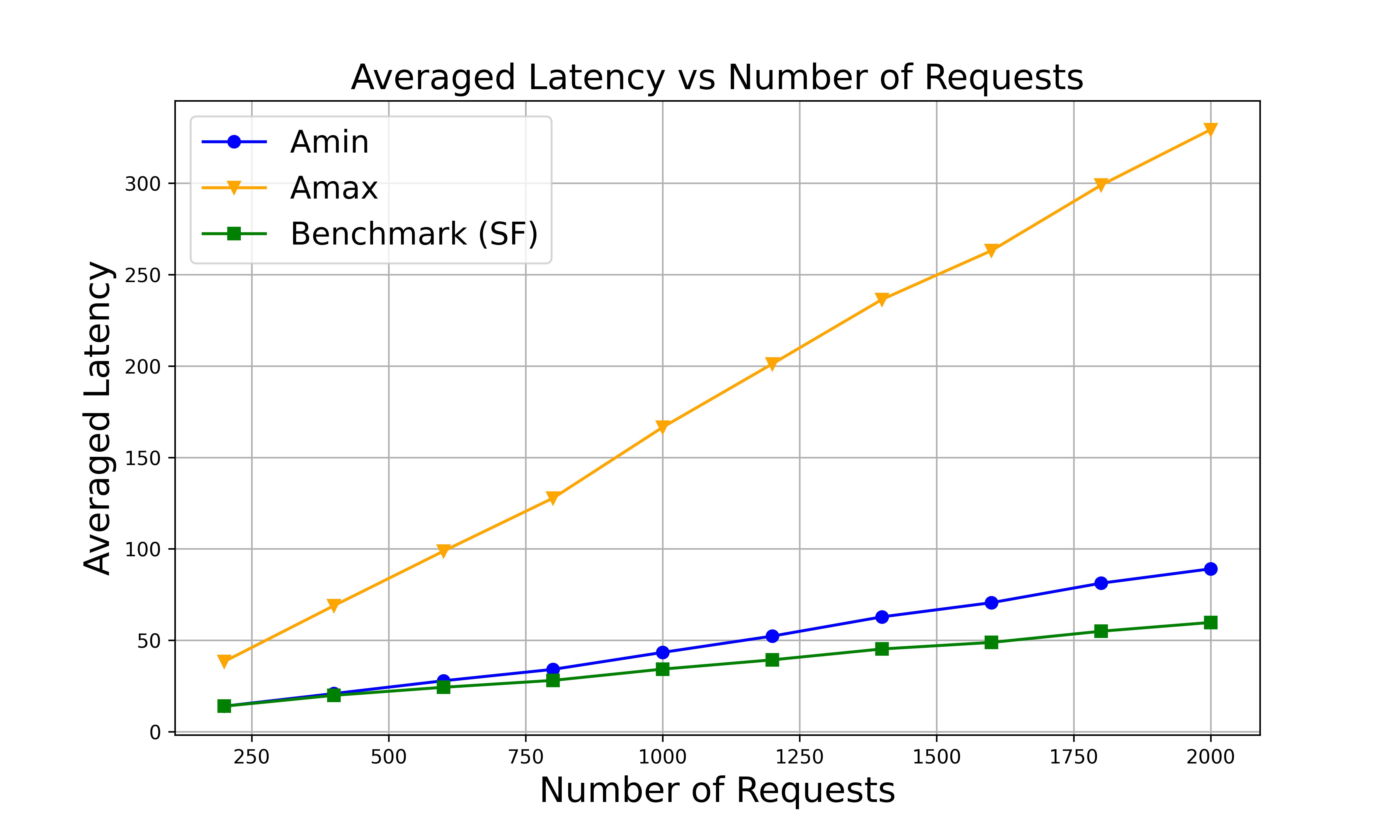}
\caption{Averaged latency between scheduling algorithms when the prediction for every request's output length is $[1,1000]$.} 
\label{fig:exp1}
\end{figure}

\xhdr{Results for Experiment 2.}
In Experiment 2, we classify requests into 10 groups, each associated with a non-overlapping output length prediction interval: $[1, 100], [101, 200], \ldots, [901, 1000]$. Compared to Experiment 1, this setting provides significantly more accurate predictions. Figure~\ref{fig:exp2} shows the average latency of the three scheduling algorithms.

First, we observe that $\mathcal{A}_{\max}$ achieves a substantial improvement over its performance in Experiment 1—the average latency is nearly halved. This is because $\mathcal{A}_{\max}$ now treats each request as having an output length equal to the upper bound of its assigned group interval, allowing more requests to be packed into each batch compared to the uniform assumption of 1000 tokens in Experiment 1.

More importantly, Figure~\ref{fig:exp2} also demonstrates that as prediction accuracy improves, the performance of $\mathcal{A}_{\min}$ closely approaches that of the benchmark \HS. This highlights a key strength of $\mathcal{A}_{\min}$: its ability to adaptively leverage better predictions to achieve latency performance nearly identical to an ideal scheduler with full information.

\begin{figure}[!ht]
\center
\includegraphics[width=0.8\textwidth]{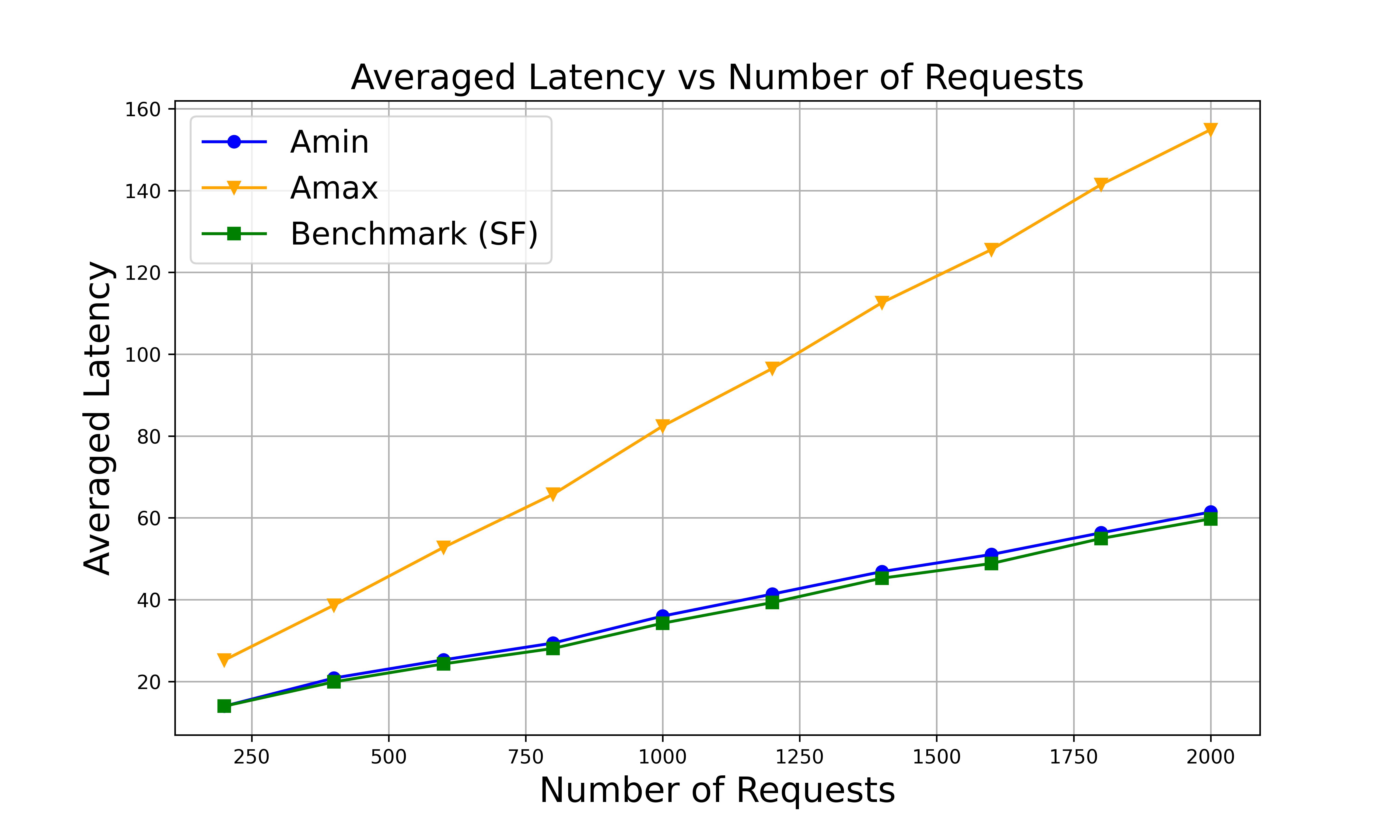}
\caption{Averaged latency between scheduling algorithms when the prediction for every request's output length is one of $[1, 100], [101, 200], \ldots, [901, 1000]$.} 
\label{fig:exp2}
\end{figure}

\xhdr{Results for Experiment 3.}
In Experiment 3, each request $i$ is assigned a prediction interval of the form $[ (1 - x)o_i, (1 + x)o_i ]$, where $x \in \{0.1, 0.95, 0.99\}$ controls the accuracy of the prediction. A smaller value of $x$ corresponds to more accurate predictions.

As shown in Figure~\ref{fig:exp3}, when $x = 0.1$, indicating highly accurate predictions, both $\mathcal{A}_{\max}$ and $\mathcal{A}_{\min}$ perform well, achieving low average latency. However, as the prediction accuracy decreases (i.e., $x = 0.95$ or $x = 0.99$), the performance of $\mathcal{A}_{\max}$ deteriorates significantly. This is because the algorithm treats each request pessimistically by assuming its output length is near the upper bound of its interval, which leads to low memory utilization and small batch sizes.

In contrast, $\mathcal{A}_{\min}$ continues to maintain low average latency even under highly uncertain predictions. Its performance remains close to the hindsight benchmark \HS, highlighting the robustness of $\mathcal{A}_{\min}$ in the face of imprecise output length estimates.

\begin{figure}[!ht]
\center
\includegraphics[width=0.8\textwidth]{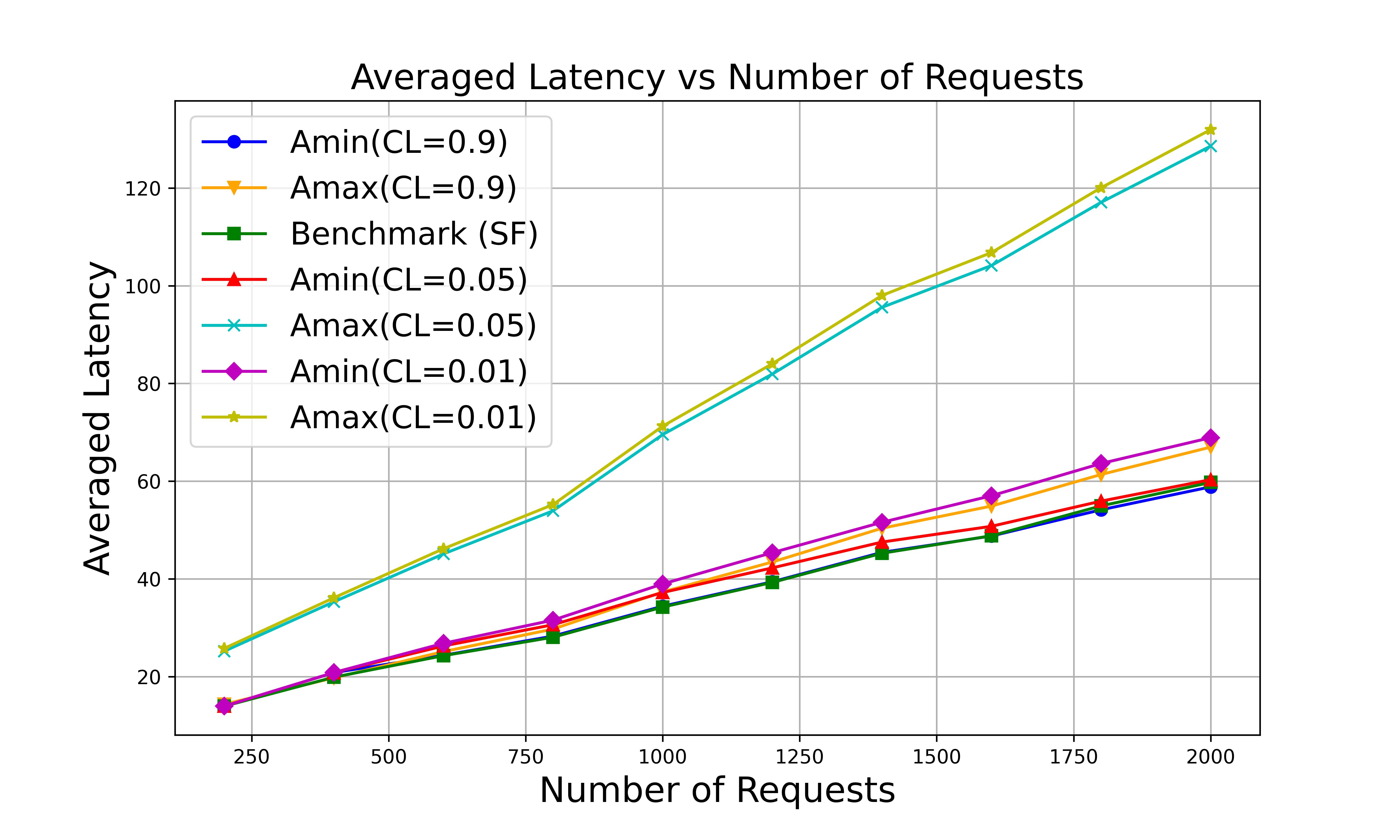}
\caption{Averaged latency between scheduling algorithms when the prediction for every request $i$'s output length is $[(1-x)o_i,(1+x)o_i]$.} 
\label{fig:exp3}
\end{figure}
\section{Conclusion} \label{sec:conclusion}

This paper extends the LLM inference scheduling model introduced by \cite{jaillet2025online} to a more realistic setting where only interval-based predictions of output lengths are available. We first analyze $\mathcal{A}_{\max}$, a conservative algorithm that avoids memory overflow by overestimating request lengths, but demonstrate its fundamental limitation: poor robustness when prediction intervals are wide, leading to excessive resource allocation and reduced concurrency. Our key contribution is $\mathcal{A}_{\min}$, an adaptive algorithm that employs strategic underestimation followed by progressive refinement during execution. Through both theoretical analysis and comprehensive numerical experiments, we establish that $\mathcal{A}_{\min}$ achieves robust performance across varying prediction quality while maintaining high system efficiency. These results provide practical insights for deploying LLM inference systems under prediction uncertainty.

\newpage

\bibliographystyle{ACM-Reference-Format}
\bibliography{reference}

\ECSwitch

\ECHead{Proofs of Statements}

\section{Memory-Preserving Proof Technique for Theorem \ref{thm:Amax}}
\label{sec:memory}

In this section, we study the relationship between the total end-to-end latency and the memory constraint \( M \). Although it is generally intractable to derive a closed-form expression for total latency as a function of \( M \), we can still make meaningful comparisons between algorithms that satisfy certain structural properties, which will be formally defined in Definitions~\ref{def:1} and~\ref{def:2}. Specifically, by analyzing how total latency scales under different memory capacities for such algorithms, we can bound their relative performance. Based on this approach, we construct a benchmark algorithm and compare its total latency with that of \( \mathcal{A}_{\max} \) to establish the competitive ratio.

To introduce the structural ideas underlying our analysis, we first define the notion of the \textit{corresponding order} between two schedules in Definition~\ref{def:1}. This concept formalizes when two schedules maintain the same relative processing order, even if their exact starting times differ.

\begin{definition}[Corresponding Order] \label{def:1}
Let \( I = ( I_0, I_1, \cdots, I_T ) \) and \( I' = ( I_0', I_1', \cdots, I_T' ) \) be two schedules produced by possibly different algorithms. We say that \( I \) and \( I' \) \textbf{preserve the corresponding order} if, for any two jobs \( n_1 \in I_i \cap I'_{i'} \) and \( n_2 \in I_j \cap I'_{j'} \), it holds that
\[
(i - j)(i' - j') \geq 0.
\]
In other words, if \( n_1 \) is scheduled earlier than \( n_2 \) in \( I \), it is also scheduled no later than \( n_2 \) in \( I' \), and vice versa. Furthermore, we say that \( I' \) is \textbf{delayed relative to} \( I \), denoted \( I \leq I' \), if for every job \( n_0 \in I_i \cap I'_{i'} \), we have \( i \leq i' \).

\end{definition}

Definition~\ref{def:1} focuses on the relative order of job processing rather than their exact starting times. This abstraction is crucial because, when comparing scheduling strategies (e.g., shortest-job-first or longest-job-first) under different memory capacities \( M \), the exact batch compositions and starting times may vary significantly. Due to the linearly increasing memory usage during processing, tracking exact timings is highly complex. However, the key observation is that for scheduling strategies based purely on job ordering, the relative order of job processing remains stable even as \( M \) changes. For instance, shortest-job-first will always prioritize jobs according to their output lengths, independent of the specific value of \( M \).

\begin{example}
Given the memory capacity \( M = 7 \) with \( n = 4 \), \( s = 1\) and \( \mathbf{o} = (1, 2, 3, 4) \). Consider the following different input schedules \( I_1 = (\{1, 3\}, \emptyset, \{2\}, \{4\})\) and \( I_2 = (\{3\}, \{1, 2\}, \emptyset, \{4\})\). We find that every request in \(I_1\) scheduled earlier does not start processing later in \(I_2\) and vice versa. This means that \(I_1\) and \(I_2\) preserve the corresponding order. However, the change of requests \(1\) and \(2\) shows that neither \(I_1\) nor \(I_2\) are delayed relative to the other. Moreover, let \( I_3 = (\{1, 2, 3\}, \emptyset, \{4\})\) which is the optimal input sequence. We can see that \(I_1, I_2, I_3\) preserve the corresponding order pairwise. In addition, because all jobs in \(I_3\) begin earlier, both \(I_1\) and \(I_2\) are delayed relative to \(I_3\), \(I_3 \le I_1, I_2 \).
\end{example}

Next, to meaningfully compare the total latency of two schedules, it is not sufficient for them to simply preserve the same corresponding order. We must also ensure that both schedules fully utilize the available memory—that is, they do not intentionally leave memory unused in a way that artificially increases latency. However, because the memory usage of each job grows linearly during processing, it is nontrivial to characterize what it means for a schedule to fully utilize memory. To formalize this notion, we introduce the concept of an \textit{optimally packed} schedule in Definition~\ref{def:2}.

\begin{definition}[Optimally Packed Schedule] \label{def:2}
Let \( M \) denote the KV cache memory limit. A schedule \( I = ( I_0, I_1, \cdots, I_T ) \) is said to be \textbf{optimally packed} if for every job \( n_0 \in I_i \) with \( i \in [T] \), the following holds:  
Let
\[
j := \max\{k \in [0,i-1] : I_k \neq \emptyset\}
\]
be the most recent time before \( i \) such that \( I_j \) is non-empty. Then, for any \( k \in [j,i) \), moving \( n_0 \) to batch \( I_k \) would result in exceeding the memory limit \( M \) at some point during processing.
\end{definition}



Definition~\ref{def:2} captures the idea that a schedule is memory-efficient in a myopic sense: a job cannot be moved earlier into any nearby batch—specifically, between its original batch and the nearest previous non-empty batch—without violating the memory constraint. The definition focuses only on local perturbations to ensure that no memory is intentionally left unused in a way that could have allowed earlier scheduling within immediate reach. This notion of optimal packing is thus well-suited for analyzing algorithms that follow natural, sequential batching processes. Combining Definitions~\ref{def:1} and~\ref{def:2}, we are now prepared to compare the total latency of two schedules that (i) preserve the same corresponding order and (ii) are both optimally packed, but operate under different memory capacities.

\begin{proposition}[Latency Scaling under Reduced Memory] \label{prop:memo_latency}
Let \( I \) and \( I' \) be two schedules that preserve the corresponding order and are both optimally packed. Suppose \( I \) satisfies a memory constraint of \( M \), while \( I' \) satisfies a memory constraint of \( \beta M \), where \( \beta \in (0,1) \). Then, the total end-to-end latencies satisfy the following inequality:
\begin{align}
    \text{TEL}(\textbf{o};I') \leq \beta^{-1} \, \text{TEL}(\textbf{o};I).
\end{align}
\end{proposition}

\proof{Proof of Proposition \ref{prop:memo_latency}}
We aim to show that, under the given memory constraints, the total end-to-end latency of the input sequence \( I' \) is at most \( \beta^{-1} \) times that of the input sequence \( I \).

First, fix the input order \( \mathbf{o} = (o_1, o_2, \dots, o_n) \).  
Assume that both schedules \( I \) and \( I' \) are optimally packed, meaning that at every time step, the system processes as many requests as permitted by the memory constraint.

Since the memory capacity in \( I' \) is scaled by a factor \( \beta \in (0,1) \) relative to that in \( I \), and each request \( o_i \) consumes the same amount of memory in both schedules, it follows that at each moment, the number of requests being processed in \( I' \) is \( \beta \) times the number in \( I \).

Let \( L_i \) and \( L_i' \) denote the completion times of request \( i \) under schedules \( I \) and \( I' \), respectively.  
Our goal is to relate \( L_i' \) to \( L_i \) for each \( i \in [n] \).

Because \( I' \) has a memory capacity reduced by a factor \( \beta \), and because both schedules are fully packed, the overall processing rate in \( I' \) is slowed down by a factor of \( \beta \) compared to \( I \).  
Thus, to complete the same workload, the time needed under \( I' \) is multiplied by a factor of \( \beta^{-1} \).
This implies that the completion time of each request satisfies
\begin{align}
    L_i' \leq \beta^{-1} L_i, \quad \forall i \in [n].
\end{align}

    Now, recall that the total end-to-end latency $\text{TEL}$ for an input sequence is defined as the sum of the completion times for all requests:
    \begin{align}
        \text{TEL}(\mathbf{o};I) = \sum_{i=1}^n L_i,
    \quad
        \text{TEL}(\mathbf{o};I') = \sum_{i=1}^n L_i'.
    \end{align}
    Substituting the earlier established relation between $L_i$ and $L_i'$, we have:
    \begin{align}
        \text{TEL}(\mathbf{o};I') = \sum_{i=1}^n L_i' \le \sum_{i=1}^n \beta^{-1} L_i = \beta^{-1} \text{TEL}(\mathbf{o};I).
    \end{align}
    Thus, we conclude that 
    \[ \text{TEL}(\mathbf{o};I') \le \beta^{-1} \, \text{TEL}(\mathbf{o};I). \]
\Halmos
\endproof

To apply Proposition~\ref{prop:memo_latency} in our analysis, we first need to construct an auxiliary algorithm that (i) preserves the same corresponding order as \( \mathcal{A}_{\max} \) and (ii) has a competitive ratio that is easier to analyze. By comparing \( \mathcal{A}_{\max} \) with this auxiliary algorithm, we can bound the performance of \( \mathcal{A}_{\max} \). Noting that \( \mathcal{A}_{\max} \) treats all jobs as identical and selects a batch by randomly sampling the maximum feasible number of requests, we construct a fully randomized benchmark algorithm, $\AR$ designed specifically for this comparison. The details of $\AR$ can be found in Algorithm \ref{algo:random} in Appendix \ref{append:memory}.

Algorithm~$\AR$ is an auxiliary algorithm constructed for analysis purposes, and it is not intended to be practically implementable. In particular, it requires full knowledge of all true output lengths $\mathbf{o} = (o_1, o_2, \ldots, o_n)$ as input, similar to a hindsight algorithm that perfectly predicts the realized output lengths of all jobs. More specifically, Algorithm~$\AR$ first randomly samples a permutation of all $n$ requests and then processes the jobs according to this fixed order. At each step, it sequentially adds requests following the permutation, ensuring that the memory constraint in Equation~\eqref{eqn:Constraint} is satisfied at all times. Since it has access to the true values of $\mathbf{o}$, the algorithm can achieve an optimally packed schedule. Next, we present a theorem which characterizes the competitive ratio of Algorithm~$\AR$ and states that for any scheduling policy, the worst case always happens when the true output length $o_i$ is either $\ell$ or $u$ for all $i \in [n]$.

\begin{theorem}[Worst-Case Realizations and Competitive Ratio of $\AR$] \label{thm:fully_random}
For any feasible scheduling policy \( \pi \), the worst-case competitive ratio is attained when each output length \( o_i \) takes either the lower bound \( \ell \) or the upper bound \( u \). That is,
\begin{align}
    \text{CR}(\pi) := \sup_{\mathbf{o} \in [\ell, u]^n} \frac{\mathbb{E}[\text{TEL}(\mathbf{o};\pi)]}{\text{TEL}(\mathbf{o};\HS)} = \sup_{\mathbf{o} \in \{\ell, u\}^n} \frac{\mathbb{E}[\text{TEL}(\mathbf{o};\pi)]}{\text{TEL}(\mathbf{o};\HS)}.
\end{align}
Moreover, the competitive ratio of the auxiliary algorithm \( \AR \) satisfies
\begin{align}
    \text{CR}(\AR) \leq \frac{1 + \alpha^{-1}}{2} + \mathcal{O}(\frac{1}{M}),
\end{align}
as \( M \to +\infty \), where \( \alpha = \ell / u \).
\end{theorem}

The proof of Theorem \ref{thm:fully_random} can be found in Appendix \ref{append:memory}. 
Next, we compare two versions of Algorithm \(\AR\): 
\begin{itemize}
\item \(\AR(M)\): Algorithm \(\AR\) run with memory capacity \(M\),
\item \(\AR(\alpha M)\): Algorithm \(\AR\) run with memory capacity \(\alpha M\), where \(\alpha = \ell/u\).
\end{itemize}

Both versions produce randomized schedules based on random permutations of the input requests. Recall that Proposition~\ref{prop:memo_latency} states that for any pair of schedules that preserve the corresponding order and are optimally packed, we can directly compare their total latencies. To formalize the connection between \(\AR(M)\) and \(\AR(\alpha M)\), we introduce the following probabilistic coupling:

\begin{definition}[Permutation Coupling]  
Sample a random permutation \( \sigma \) of \( [n] \) uniformly at random. 
\begin{itemize}
\item Under \(\AR(M)\), schedule the jobs according to \( \sigma \), subject to memory capacity \(M\).
\item Under \(\AR(\alpha M)\), schedule the jobs according to the same permutation \( \sigma \), subject to memory capacity \(\alpha M\).
\end{itemize}
\end{definition}

By this probabilistic coupling, each random schedule generated by \(\AR(M)\) is paired with a corresponding schedule generated by \(\AR(\alpha M)\), and each pair occurs with the same probability.
Now, since the coupled schedules preserve the same corresponding order and are optimally packed, Proposition~\ref{prop:memo_latency} implies that the total latency of each realization under \(\AR(\alpha M)\) is at most \(\alpha^{-1}\) times the corresponding latency under \(\AR(M)\).  
Taking expectations over the coupling, we conclude that the expected total latency under \(\AR(\alpha M)\) is at most \(\alpha^{-1}\) times the expected total latency under \(\AR(M)\), namely,
\begin{align} \label{eq:randomcompare1}
    \mathbb{E}[\text{TEL}(\mathbf{o};\AR(\alpha M))] \le \alpha^{-1}\mathbb{E}[\text{TEL}(\mathbf{o};\AR(M))].
\end{align}

As Theorem~\ref{thm:fully_random} establishes the competitive ratio of \(\AR(M)\), to bound the competitive ratio of \(\mathcal{A}_{\max}\), it suffices to compare the expected total latency between \(\mathcal{A}_{\max}\) and \(\AR(M)\). By Equation~\eqref{eq:randomcompare1}, we can introduce \(\AR(\alpha M)\) as an intermediate benchmark. The next lemma shows that it is easier to compare the expected total latency between \(\mathcal{A}_{\max}\) with memory capacity \(M\) and \(\AR(\alpha M)\).

\begin{lemma}
\label{lem:max<random}
    Denote \(\AR(\alpha M)\) as Algorithm \(\AR\) run with memory capacity \(\alpha M\), where $\alpha = \frac{\ell}{u}$, then we have 
    \begin{align} \label{eq:randomcompare2}
    \mathbb{E}[\text{TEL}(\mathbf{o};\mathcal A_{\max})] \le \mathbb{E}[\text{TEL}(\mathbf{o};\AR(\alpha M))].
    \end{align}
\end{lemma}

The proof of Lemma \ref{lem:max<random} can be found in Appendix \ref{append:memory}. Combining the result in Lemma \ref{lem:max<random} and Equation~\eqref{eq:randomcompare2}, we are ready to give the competitive ratio upper bound to $\mathcal A_{\max}$, which states in Theorem \ref{thm:Amax}. The proof can also be found in Appendix \ref{append:memory}.

\section{Supplementary Materials for Section \ref{sec:benchmark}} \label{append:benchmark}

\begin{algorithm}
\label{algo:off}
\caption{\HS}
\KwIn{Memory Capacity $M$, prompt requests $\mathbf{o}=(o_1,o_2,\ldots,o_n)$.}
\KwOut{Processing Sequence: $I( I_0, I_1, \cdots, I_T)$.}

\While{there exist waiting requests}{
    Let $R_t$ be the set of waiting prompts at time $t$. Let $S_t$ be the set of tokens currently processing at time $t$.
    
    Find the set $I_t \subset R_t$ with largest cardinality such that Equation \eqref{eq:constraintM} hold for all $t' \geq t$.

    Process the requests in $I_t \cup S_t$ and update $R_{t+1} = R_t \backslash I_t$.
    }
\end{algorithm}

\section{Supplementary Materials for Section \ref{sec:memory}} \label{append:memory}

\proof{Proof of Theorem \ref{thm:fully_random}}
We split the proof into two parts, where in part 1, we show that the worst case happens only if the output length of each $o_i$ is either $\ell$ or $u$. In part 2, we prove for the upper bound of the competitive ratio of $\AR$.

\textit{---Part 1---}

Without loss of generality, assume the output lengths are ordered as \( o_1 \leq o_2 \leq \cdots \leq o_n \).  
We analyze how the total end-to-end latency changes when the output length of a single request is increased or decreased by one unit.

First, denote \( \pi \) as the scheduling policy \( \mathcal{A}_{\text{random}} \) and fix an input vector \( \mathbf{o} = (o_1, \ldots, o_n) \). Consider decreasing the output length of job \( i \) from \( o_i \) to \( o_i - 1 \), producing a new instance \( \mathbf{o}^- \). Because the job order remains unchanged, the start time of job \( i \) is unaffected. The main effect is that some requests scheduled after job \( i \) may now complete earlier, potentially reducing their latency by one unit.

Let \( F \) denote the set of jobs whose completion times are advanced by one unit due to this reduction. Then, the total latency under \( \mathbf{o}^- \) satisfies
\[
\mathbb{E}[\text{TEL}(\mathbf{o}^-;\pi)] = \mathbb{E}[\text{TEL}(\mathbf{o};\pi)] - 1 - |F|.
\]

Similarly, consider increasing \( o_i \) to \( o_i + 1 \), producing a new instance \( \mathbf{o}^+ \). Let \( B \) denote the set of jobs whose completion times are delayed by one unit. Then,
\[
\mathbb{E}[\text{TEL}(\mathbf{o}^+;\pi)] = \mathbb{E}[\text{TEL}(\mathbf{o};\pi)] + 1 + |B|.
\]

Since the scheduling policy \( \pi \) does not know the true output lengths, and because jobs with the same output size are symmetric under random permutations, we can pair the effects of decreasing and increasing a job’s output length. Taking expectations yields:
\[
\mathbb{E}[\text{TEL}(\mathbf{o}^-;\pi)] + \mathbb{E}[\text{TEL}(\mathbf{o}^+;\pi)] = 2 \, \mathbb{E}[\text{TEL}(\mathbf{o};\pi)].
\]

Now consider the hindsight shortest-job-do-first's latency \( \text{TEL}(\mathbf{o};\HS) \) for each instance. Since the hindsight policy knows all output lengths in advance, changing the output length of a single job affects only the scheduling of subsequent jobs. Moreover, because jobs with smaller output lengths consume less memory and are generally processed earlier under \HS, reducing a job’s output length tends to bring forward more jobs than increasing it delays. This gives
\[
\text{TEL}(\mathbf{o}^-;\HS) + \text{TEL}(\mathbf{o}^+;\HS) \leq 2 \, \text{TEL}(\mathbf{o};\HS).
\]

Using these observations, we can now compare the competitive ratios:
\[
\frac{\mathbb{E}[\text{TEL}(\mathbf{o};\pi)]}{\text{TEL}(\mathbf{o};\HS)} \leq \max\left\{ \frac{\mathbb{E}[\text{TEL}(\mathbf{o}^-;\pi)]}{\text{TEL}(\mathbf{o}^-;\HS)}, \, \frac{\mathbb{E}[\text{TEL}(\mathbf{o}^+;\pi)]}{\text{TEL}(\mathbf{o}^+;\HS)} \right\}.
\]
This inequality shows that either decreasing or increasing the output length of a single job by one unit leads to a higher or equal competitive ratio.

Finally, suppose there exists an instance \( \mathbf{o} \in [\ell, u]^n \) achieving the supremum in the definition of the competitive ratio. If there exists a job \( i \) such that \( \ell < o_i < u \), then by perturbing \( o_i \) to \( o_i-1 \) or \( o_i+1 \) (and reordering the jobs if necessary), we can produce an instance with a higher or equal competitive ratio. Repeating this process finitely many times eventually produces an instance where each \( o_i \in \{ \ell, u \} \).

Therefore,
\[
\sup_{\mathbf{o} \in [\ell, u]^n} \frac{\mathbb{E}[\text{TEL}(\mathbf{o};\pi)]}{\text{TEL}(\mathbf{o};\HS)} = \sup_{\mathbf{o} \in \{ \ell, u \}^n} \frac{\mathbb{E}[\text{TEL}(\mathbf{o};\pi)]}{\text{TEL}(\mathbf{o};\HS)},
\]
completing the proof.

\textit{---Part 2---}

By Part 1, suppose there are \( N_{\ell} \) requests with output length \( \ell \) and \( N_u \) requests with output length \( u \), where \( N_{\ell} \) and \( N_u \) are large. By definition, the expected total latency under Algorithm \( \AR \) can be decomposed as
\[
\mathbb{E}[\text{TEL}(\mathbf{o};\AR)] = N_{\ell} \, \mathbb{E}[L_i \mid o_i = \ell] + N_u \, \mathbb{E}[L_j \mid o_j = u],
\]
where \( L_i \) and \( L_j \) denote the end-to-end latencies of requests with output lengths \( \ell \) and \( u \), respectively.

Since Algorithm \( \AR \) generates a random permutation and processes requests sequentially while ensuring optimal packing, the schedule is optimally packed by construction. The total amount of work, measured in terms of memory consumption over time, is upper bounded by
\[
\ell \lceil \frac{N_\ell}{\lfloor {\frac{M}{s+\ell}} \rfloor} \rceil + u \lceil \frac{N_u}{\lfloor {\frac{M}{s+u}} \rfloor} \rceil \le
\left( 1 + \mathcal{O}(\frac{1}{M}) \right) \left( \frac{N_{\ell}}{M} \, \ell (s+\ell) + \frac{N_u}{M} \, u (s+u) \right),
\]
where \( M \) is the memory capacity, and \( s+\ell \) and \( s+u \) represent the peak memory footprint for requests of type \( \ell \) and \( u \), respectively.

Due to the symmetry of random sampling, each request of a given type has the same expected latency. Furthermore, the expected latency for each request is at most half of the total processing time for all requests, because in expectation, a request will be located near the middle of the cumulative schedule. Thus, we have
\[
\mathbb{E}[L_i \mid o_i = \ell], \, \mathbb{E}[L_j \mid o_j = u] \leq \frac{1}{2} \left( 1 + \mathcal{O}(\frac{1}{M}) \right) \left( \frac{N_{\ell}}{M} \, \ell (s+\ell) + \frac{N_u}{M} \, u (s+u) \right).
\]

Next, consider the hindsight shortest-job-do-first's total latency \( \text{TEL}(\mathbf{o};\HS) \). Since the offline scheduler knows all output lengths in advance and it can prioritize smaller jobs first, requests with output length $l$ have latency at least \( N_{\ell} \frac{1}{2} \ell \lfloor \frac{N_\ell}{\lceil {\frac{M}{s+\ell}} \rceil} \rfloor\). Also, requests with output length $u$ are subject to a delay of \( \ell \lfloor \frac{N_\ell}{\lceil {\frac{M}{s+\ell}} \rceil} \rfloor \) periods, and their latency is more than \( N_{\ell} \left( \ell \lfloor \frac{N_\ell}{\lceil {\frac{M}{s+\ell}} \rceil} \rfloor + \frac{1}{2} u \lfloor \frac{N_u}{\lceil {\frac{M}{s+u}} \rceil} \rfloor \right) \). The corresponding total latency is at least
\[
\text{TEL}(\mathbf{o};\HS) \geq \frac{1}{2} \left( 1 + \mathcal{O}(\frac{1}{M}) \right)^{-1} \left( l (s+\ell) \frac{N_{\ell}^2}{M} + 2 l (s+\ell) \frac{N_{\ell} N_u}{M} + u (s+u) \frac{N_u^2}{M} \right).
\]

Taking the ratio, we obtain
\[
\frac{\mathbb{E}[\text{TEL}(\mathbf{o};\AR)]}{\text{TEL}(\mathbf{o};\HS)} \leq \left( 1 + \mathcal{O}(\frac{1}{M}) \right) \frac{(N_{\ell} + N_u) \left( N_{\ell} \ell (s+\ell) + N_u u (s+u) \right)}{l(s+\ell)N_{\ell}^2 + 2l(s+\ell)N_{\ell}N_u + u(s+u)N_u^2}.
\]

Since \( \alpha = \ell/u \), simplifying the expression yields 
\[
\frac{\mathbb{E}[\text{TEL}(\mathbf{o};\AR)]}{\text{TEL}(\mathbf{o};\HS)} \leq \left( 1 + \mathcal{O}(\frac{1}{M}) \right) \left( 1 + \frac{1 - \alpha^2}{2(\alpha + \alpha^2)} \right) = \frac{1+\alpha^{-1}}{2} + \mathcal{O}(\frac{1}{M}).
\]

Thus, the competitive ratio of \( \AR \) is at most \( \frac{1+\alpha^{-1}}{2} + \mathcal{O}(\frac{1}{M}) \), as claimed.

\Halmos
\endproof

\proof{Proof of Lemma~\ref{lem:max<random}}
We first define a probabilistic coupling between \(\mathcal{A}_{\max}\) and \(\AR(\alpha M)\):  

\begin{definition}[Permutation Coupling for \(\mathcal{A}_{\max}\) and \(\AR(\alpha M)\)]  
Sample a random permutation \( \sigma \) of \([n]\) uniformly at random.  
\begin{itemize}
\item Under \(\mathcal{A}_{\max}\), schedule the jobs according to \(\sigma\), treating all output lengths as \(u\) and applying memory constraint \(M\).
\item Under \(\AR(\alpha M)\), schedule the jobs according to the same permutation \(\sigma\), using the true output lengths and memory constraint \(\alpha M\).
\end{itemize}
\end{definition}

By construction, under this coupling, the two schedules share the same processing order.

Denote by \( I_{\mathcal{A}_{\max}} \) the input sequence generated by \(\mathcal{A}_{\max}\) with memory \(M\), and \( I_{\AR(\alpha M)} \) the input sequence generated by \(\AR(\alpha M)\) with memory \(\alpha M\).  
We claim that, under the coupling,
\[
I_{\mathcal{A}_{\max}} \leq I_{\AR(\alpha M)},
\]
meaning that every request in \( \mathcal{A}_{\max} \) is delayed to \( \AR(\alpha M) \). The definition of one schedule is delayed to the other is defined in Definition \ref{def:1}.

The reasoning is as follows.
At any time, if there remains sufficient memory (at least \( s+u \)) under \(\mathcal{A}_{\max}\), the algorithm immediately initiates a new request, treating it as requiring output length \(u\).  
Since \(u\) is the worst-case (largest) output length and \(\mathcal{A}_{\max}\) batches jobs based on this conservative estimate, the memory utilization rate per job under \(\mathcal{A}_{\max}\) is at least \(\frac{s+\ell}{s+u} \geq \alpha\).

In contrast, under \(\AR(\alpha M)\), the algorithm uses the true output lengths, which can be as small as \(\ell\), and thus may utilize the memory less aggressively.  
Therefore, \(\mathcal{A}_{\max}\) tends to process more requests earlier compared to \(\AR(\alpha M)\).  
As a result, for every request, the number of preceding completed jobs is at least as large under \( \mathcal{A}_{\max} \) as under \( \AR(\alpha M) \), leading to earlier or equal start times.

Consequently, we have
\[
\text{TEL}(\mathbf{o};I_{\mathcal{A}_{\max}}) \leq \text{TEL}(\mathbf{o};I_{\AR(\alpha M)})
\]
for each realization.

Taking expectations over the randomness in the permutation coupling, we conclude
\[
\mathbb{E}[\text{TEL}(\mathbf{o};\mathcal{A}_{\max})] \leq \mathbb{E}[\text{TEL}(\mathbf{o};\AR(\alpha M))],
\]
which completes the proof.
\Halmos
\endproof

\proof{Proof of Theorem \ref{thm:Amax}}
By Lemma \ref{lem:max<random}, we have 
\begin{align}
    \mathbb{E}[\text{TEL}(\mathbf{o};\mathcal A_{\max})] \le \mathbb{E}[\text{TEL}(\mathbf{o};\AR(\alpha M))].
\end{align}

In addition, we define a probabilistic coupling between \(\AR(M)\) and \(\AR(\alpha M)\).
\begin{definition}[Permutation Coupling for \(\AR(M)\) and \(\AR(\alpha M)\)]  
Sample a random permutation \( \sigma \) of \([n]\) uniformly at random.  
\begin{itemize}
\item Under \(\AR(M)\), schedule the jobs according to the same permutation \(\sigma\), using the true output lengths and memory constraint \(M\).
\item Under \(\AR(\alpha M)\), schedule the jobs according to the same permutation \(\sigma\), using the true output lengths and memory constraint \(\alpha M\).
\end{itemize}
\end{definition}

Denote by \( I_{\AR(M)} \) the input sequence generated by \(\AR(M)\) with memory \(M\), and \( I_{\AR(\alpha M)} \) the input sequence generated by \(\AR(\alpha M)\) with memory \(\alpha M\). By construction, we see that they preserve the corresponding order and are both optimally packed. Recall Proposition \ref{prop:memo_latency}, we have
\begin{align}
    \text{TEL}(\textbf{o};I_{\AR(M)}) \leq \alpha^{-1} \, \text{TEL}(\textbf{o};I_{\AR(\alpha M)}).
\end{align}
Taking expectations over the randomness in the permutation coupling, we arrive at
\[
\mathbb{E}[\text{TEL}(\mathbf{o};\AR(\alpha M))] \leq \alpha^{-1} \mathbb{E}[\text{TEL}(\mathbf{o};\AR(M))].
\]

Combining two results, we have 
\[
\mathbb{E}[\text{TEL}(\mathbf{o};\mathcal A_{\max})] \leq \alpha^{-1} \mathbb{E}[\text{TEL}(\mathbf{o};\AR(M))].
\]
Taking the competitive ratio, we obtain
\[
\text{CR}(\mathcal{A}_{\max}) \leq \alpha^{-1} \text{CR}(\mathcal{\AR}).
\]

According to Theorem \ref{thm:fully_random}, the competitive ratio of the algorithm \( \mathcal{A}_{\max} \) satisfies
\begin{align}
    \text{CR}(\mathcal{A}_{\max}) \leq \alpha^{-1} \text{CR}(\mathcal{\AR}) \leq \frac{\alpha^{-1} (1+\alpha^{-1})}{2} + \mathcal{O}(\frac{1}{M}),
\end{align}
as \( M \to +\infty \), which is the desired conclusion.
\Halmos
\endproof

\begin{algorithm}[H]
\label{algo:random}
\caption{Auxiliary $\AR$}
\KwIn{Memory Capacity $M$, prompt size $s$, output length $\mathbf{o} = (o_1, o_2, \ldots, o_n)$.}
\KwOut{Processing sequence $I = ( I_0, I_1, \cdots, I_T)$.}

Randomly generate a permutation \( (i_1, \ldots, i_n) \) of \( [n] \).

\While{there exist waiting requests}{
    Let \( R_t \) be the set of waiting prompts at time \( t \). Let \( S_t \) be the set of tokens currently processing at time \( t \).

    Following the permutation order, successively add requests from \( R_t \) into the memory until no further addition is possible without violating Equation~\eqref{eqn:Constraint} at any \( t' \geq t \).

    Let \( I_t \subset R_t \) be the selected set of prompts.

    Process the requests in \( I_t \cup S_t \) and update \( R_{t+1} = R_t \setminus I_t \).
}
\end{algorithm}

\section{Supplementary Materials for Section \ref{sec:algrobust}} \label{append:algrobust}

\proof{Proof of Theorem~\ref{thm:complexity}}

We analyze the computational complexity of Algorithm~\(\mathcal{A}_{\min}\) step by step.

First, at each time step, to select a batch of new requests, the algorithm sorts the available prompts based on the auxiliary values \( \tilde{o}_i \). Since the number of concurrent prompts is at most proportional to the memory capacity \( M \) (specifically, at most \( M/s \)), sorting requires at most \(\mathcal{O}(M \log M)\) operations.

Second, when verifying whether adding a request would violate the memory constraint, the algorithm must compute the projected memory usage at the next time step. Simulating the projected memory usage across all currently active requests requires \(\mathcal{O}(M)\) time, since the number of active requests is again at most proportional to \(M\).

Third, if a memory overflow is predicted, the algorithm must remove requests to bring memory usage back within the constraint. Since requests are already maintained in sorted order by \( \tilde{o}_i \), selecting and removing jobs can be performed efficiently. The removal operation involves traversing the sorted list and updating memory usage estimates, which takes at most \(\mathcal{O}(M)\) time overall.

Finally, updates to the auxiliary variables \( \tilde{o}_i \) occur when output tokens are generated. These updates involve only local modifications and require \(\mathcal{O}(1)\) time per token, for a total of at most \(\mathcal{O}(M)\) operations per time step.

Combining all steps, the dominant operation per time step is the initial sorting, leading to an overall computational complexity of \(\mathcal{O}(M \log M)\).

\Halmos
\endproof

\proof{Proof of Theorem \ref{thm:optimality}}
We proceed by contradiction. Suppose there exists a feasible policy \( \pi \) that achieves a strictly smaller competitive ratio than \( \mathcal{A}_{\min} \). Then, there must exist a time \( t \) at which \( \pi \) and \( \mathcal{A}_{\min} \) make different decisions for the first time.

We first consider the case in which they choose to add different requests. By definition, \( \mathcal{A}_{\min} \) always adds the request with the smallest pseudo-output length \( \tilde{o} \). Suppose \( \pi \) chooses to start request \( i \), while \( \mathcal{A}_{\min} \) starts request \( j \), with \( \Tilde{o}_i > \Tilde{o}_j \). Then, from the property of conditional expectation, we have
\[
\mathbb{E}[o_i \mid o_i \ge \Tilde{o}_i] > \mathbb{E}[o_j \mid o_j \ge \Tilde{o}_j].
\]
Let the completion time of request \( j \) under policy \( \pi \) be \( t' \). If we modify the schedule so that request \( j \) is started at time \( t \) and request \( i \) is delayed to complete at time \( t' \), then the expected net change in latency is
\[
\mathbb{E}[o_i \mid o_i \ge \Tilde{o}_i] - \mathbb{E}[o_j \mid o_j \ge \Tilde{o}_j],
\]
which is strictly positive, implying that such a switch would reduce total expected latency.

Next, consider the case in which they choose to delete different requests at time \( t \). Again, recall that \( \mathcal{A}_{\min} \) always deletes the request with the smallest \(\tilde{o}\). Suppose \( \pi \) deletes request \( i \), and \( \mathcal{A}_{\min} \) deletes request \( j \), with \( \Tilde{o}_i > \Tilde{o}_j \). This implies that request \( i \) has been processed for a longer time than request \( j \). In this case, we observe that
\[
\mathbb{E}[o_i \mid o_i \ge \Tilde{o}_i] - \Tilde{o}_i \le \mathbb{E}[o_j \mid o_j \ge \Tilde{o}_j] - \Tilde{o}_j.
\]
Moreover, the deleted request must be restarted, resulting in increased latency. The expected latency difference between the two policies is given by
\[
\left( (\mathbb{E}[o_j \mid o_j \ge \Tilde{o}_j] - \Tilde{o}_j) + \mathbb{E}[o_i \mid o_i \ge \Tilde{o}_i] \right) - \left( (\mathbb{E}[o_i \mid o_i \ge \Tilde{o}_i] - \Tilde{o}_i) + \mathbb{E}[o_j \mid o_j \ge \Tilde{o}_j] \right) \ge 0,
\]
which simplifies to \( \Tilde{o}_i - \Tilde{o}_j > 0 \). Hence, switching the deletion from \( i \) to \( j \) would again yield a lower expected latency.

Finally, we exclude the case in which only one of the two algorithms performs an action at time~\(t\), while the other remains idle.  
By construction, \( \mathcal{A}_{\min} \) always acts whenever an insertion or deletion is feasible, since its objective is to exploit as much useful information as possible under capacity constraints.  
Hence, if policy \( \pi \) attempts an insertion earlier than \( \mathcal{A}_{\min} \), the action cannot increase information gain: the inserted request will either need to be deleted prematurely or will consume memory without benefit.  
If, on the other hand, \( \pi \) inserts later than \( \mathcal{A}_{\min} \), then idle slots arise and the system is no longer optimally packed.  
Similarly, if \( \pi \) deletes a request earlier than \( \mathcal{A}_{\min} \), then the partially processed work is wasted and no additional information is gained; deleting later than \( \mathcal{A}_{\min} \) exceeds the memory budget, which necessarily deteriorates performance.

Combining both cases, we reach a contradiction to the assumption that \( \pi \) outperforms \( \mathcal{A}_{\min} \). Therefore, we conclude that
\[
\mathbb{E}[\text{TEL}(\mathbf{o};\mathcal{A}_{\min})] \le \mathbb{E}[\text{TEL}(\mathbf{o};\pi)].
\]
Taking the ratio of both sides, it follows that
\[
\text{CR}(\mathcal{A}_{\min}) \le \text{CR}(\pi).
\]

\Halmos
\endproof

\proof{Proof of Lemma \ref{lem:deletion}}
Since in period $\tau_i$ all requests are treated as having a lower bound $\Tilde{o} = i$, the symmetry among inputs with true length $o = j$ implies that only a fraction $\frac{s+i}{s+j}$ can be completed without deletion. That is,
\[
\mathbb{E} [c_{ij} | b_{ij}] = \frac{s+i}{s+j} b_{ij}.
\]
Similarly, for deletions, we obtain the following,
\[
\mathbb{E} [d_{ij} | b_{ij}] = \left(1 - \frac{s+i}{s+j}\right) b_{ij}.
\]
\Halmos
\endproof

\proof{Proof of Proposition \ref{prop:state_transition}}
We proceed by induction. In the initial period $\tau_{\ell}$, all requests are processed directly, and the result follows from Lemma~\ref{lem:deletion}. In the next period $\tau_{\ell+1}$, all remaining requests with $o = \ell + 1$ must be completed, yielding $\mathbb{E}[c_{\ell+1, \ell+1}] = \mathbb{E}[d_{\ell, \ell+1}] = \frac{1}{s+\ell+1} x_{\ell+1} n$. 

Since the deletion operation preserves the distribution, Equation~\ref{eq:update_distribution} implies that
\[
\mathbb{E}[c_{\ell+1, k}] = \frac{1}{s+\ell+1} \cdot \frac{s+\ell+1}{s+k} x_k n = \frac{1}{s+k} x_k n.
\]
This suggests that in each period, $\mathbb{E}[c_{jj}] = \frac{1}{s+j} x_j n$, and the remaining requests with $o = j+1$ satisfy $\mathbb{E}[c_{j+1, j+1}] = \frac{1}{s+j+1} x_{j+1} n$, which will be fully completed in the subsequent period. By the distribution preserving property, the result follows.
\Halmos
\endproof

\proof{Proof of Corollary \ref{cor:A/B}}
We first calculate the numerator:
\[
\sum_{k=1}^{u} \left( \sum_{i=k}^{u} i x_i \right) \left( \sum_{i=k}^{u} \frac{1}{i}x_i + 2 \sum_{i=k+1}^{u} \frac{i-k}{i} x_i \right).
\]
The coefficient of the cross term $x_i x_j$ (for $i \ne j$) is
\begin{align*}
    &\frac{1}{ij} \sum_{k=1}^{\min(i,j)} \left( i^2(1 + 2j - 2k) + j^2(1 + 2i - 2k) \right) \\ 
    &= \frac{1}{ij} \left( \min(i,j)(i^2 + j^2) + 2\min(i,j)ij(i + j) - (i^2 + j^2)\min(i,j)(\min(i,j) + 1) \right) \\
    &= \frac{\min(i,j)^2}{ij} \left( (i + j)^2 - 2\min(i,j)^2 \right) \\
    &= 2 a_{ij}.
\end{align*}
In the case $i = j$, we have $i^2 x_i^2 = a_{ii} x_i^2$.

For the denominator, the coefficient of $x_i x_j$ for $i \ne j$ is
\[
2 \min(i,j)^2 = 2 b_{ij},
\]
and for $i = j$ we have $i^2 = b_{ii}$.

In conclusion, the competitive ratio can be written as a Rayleigh quotient,
\[
\text{CR}(\mathcal{A}_{\min}) = \max_{\vec{x}}\frac{\vec{x}^\top \mathbf{A}_u \vec{x} }{\vec{x}^\top \mathbf{B}_u \vec{x}} + \mathcal{O}(\tfrac{1}{M}).
\]
\Halmos
\endproof

\proof{Proof of Lemma \ref{lem:positive-definite}}
We begin with $\mathbf{B}_u$. It is known that the matrix $\mathbf{M} = (\min(i,j))_{u \times u}$ is positive definite. Since $\mathbf{B}_u$ is the Hadamard square of $M$, i.e., $B = \mathbf{M} \circ \mathbf{M}$, and Hadamard powers of positive definite matrices preserve positive definiteness, we conclude that $\mathbf{B}_u \succ 0$.

To prove that $\mathbf{A}_u$ is also positive definite, we define a normalized matrix $\mathbf{A}_u' = (a_{ij}')_{u \times u}$,
\[
a_{ij}' := \frac{2}{ij} \cdot a_{ij} - 1.
\]
A direct calculation shows:
\[
a_{ij}' = 2 \cdot \frac{\min(i,j)}{\max(i,j)} - \left( \frac{\min(i,j)}{\max(i,j)} \right)^2.
\]
Let \( r_{ij} := \frac{\min(i,j)}{\max(i,j)} \). Then:
\[
a_{ij}' = 2r_{ij} - r_{ij}^2 = g(\log i - \log j),
\]
where we use the identity
\[
r_{ij} = \exp(-|\log i - \log j|),
\]
and define
\[
g(d) := 2e^{-|d|} - e^{-2|d|}.
\]

We claim that \( g(d) \) is a positive definite function on \( \mathbb{R} \). Since \( g \) is even and continuous, by Bochner’s theorem, it suffices to verify that its Fourier transform is nonnegative. Computing:
\[
\widehat{g}(\xi) = \int_{-\infty}^{\infty} g(d) e^{-i\xi d} \, \mathrm{d}d 
= 4 \cdot \frac{1}{1+\xi^2} - 4 \cdot \frac{1}{4 + \xi^2}
= \frac{12}{(1+\xi^2)(4+\xi^2)} > 0 \quad \forall \xi \in \mathbb{R}.
\]
Hence \( g \) is positive definite.

By Schoenberg's theorem, for any finite set of real numbers \( \{z_1, \dots, z_u\} \subset \mathbb{R} \), the matrix \( (g(z_i - z_j))_{u \times u} \) is positive semi-definite. Moreover, if the points \( z_i \) are mutually distinct and \( g(0) > 0 \), then the matrix is strictly positive definite.

It follows that the matrix
\[
\mathbf{A}_u' = (a_{ij}')_{u \times u} = (g(\log i - \log j))_{u \times u}
\]
is strictly positive definite.

Since \( \mathbf{A}_u = \frac{ij}{2}(a_{ij}' + 1) \), and the factor \( \frac{ij}{2} > 0 \), the transformation preserves positive definiteness. Therefore, \( \mathbf{A}_u \succ 0 \).
\Halmos
\endproof

\begin{algorithm}
\caption{Revised $\mathcal{A}_{\min}$ under Multiple Prediction Intervals}   
\label{algo:milbrevise}
\KwIn{Memory capacity $M$, prompt size $s$, output length lower bounds $\ell_j=\max\{\ell_k | \ell_k \leq o_i, k \in [m]\}$ for all $i \in [n]$.}
\KwOut{Processing sequence $I = ( I_0, I_1, \cdots, I_T )$.}

Initialize $\tilde{o}_i \gets \ell_j$ for all $i \in [n]$.

\While{there exist unfinished requests}{
    Let $R_t$ be the set of waiting prompts at time $t$. Let $S_t$ be the set of tokens currently processing at time $t$.

    \If{projected memory usage of $S_t$ at time $t+1$ exceeds $M$}{
        Sort $S_t$ in ascending order of $\tilde{o}_i$; break ties uniformly at random.

        Remove jobs one by one from $S_t$ (following the sorted order) until projected memory usage at $t+1$ satisfies the memory constraint.

        For each removed request $i$, update $\tilde{o}_i \gets$ number of tokens already generated by request $i$.
    }

    Let $S_t'$ be the set of remaining active requests after any removals.

    Sort $R_t$ in ascending order of $\tilde{o}_i$; break ties uniformly at random.

    Select the largest subset $I_t \subset R_t$ (in order) such that adding $I_t$ to $S_t'$ satisfies the memory constraint in Equation~\eqref{eqn:Constraint}.

    Process the requests in $I_t \cup S_t'$.

    Update $R_{t+1} = R_t \setminus I_t$.
}
\end{algorithm}

\section{Supplementary Materials for Section \ref{sec:extension}} \label{append:extension}

\proof{Proof of Theorem \ref{thm:no}}
Let $\vec{x}=(x_\ell,0,\dots,0,x_u)$ with $0<\ell<u$ and set $s=0$.
Since $x_k=0$ unless $k\in\{\ell,u\}$, Theorem~\ref{thm:Amin} yields
\[
\begin{aligned}
\text{num}
&=\ell(\ell x_\ell+u x_u)\!\left(x_\ell+\frac{\ell}{u}x_u
      +2\frac{u-\ell}{u}x_u\right)
  +u^2x_u\!\left(\frac{u-\ell}{u}\right)^2x_u,\\[4pt]
\text{den}
&=\ell^{2}x_\ell^{2}+2\ell^{2}x_\ell x_u+u^{2}x_u^{2}.
\end{aligned}
\]

Set $\alpha=\ell/u\in(0,1)$ and $t=x_\ell/x_u$ (with $t\ge0$). Dividing numerator and denominator yields:
\[
\text{CR}(\mathcal{A}_{\min})
  =1+\frac{\alpha(1-\alpha^{2})\,t}
          {\alpha^{2}t^{2}+2\alpha^{2}t+1}
     +\mathcal{O}\!\bigl(\tfrac1M\bigr).
\]

Since $(\alpha t-1)^2\ge0$ implies
$\alpha^{2}t^{2}+2\alpha^{2}t+1\ge2\alpha(1+\alpha)t$, we obtain:
\[
\frac{\alpha(1-\alpha^{2})\,t}
     {\alpha^{2}t^{2}+2\alpha^{2}t+1}
\le\frac{1-\alpha}{2}\quad(\forall\,t\ge0).
\]

Therefore,
\[
\text{CR}(\mathcal{A}_{\min})
  \le 1+\frac{1-\alpha}{2}
  +\mathcal{O}\!\bigl(\tfrac1M\bigr)
  =\frac{3-\alpha}{2}
  +\mathcal{O}\!\bigl(\tfrac1M\bigr),
\]
which completes the proof.
\Halmos
\endproof

\begin{algorithm}[ht!]
\caption{$\mathcal{A}_\ell$}
\label{algo:promote_l}
\KwIn{Memory capacity $M$, prompt size $s$, output lengths $o_i \in \{\ell, u\}$ for all $i \in [n]$.}
\KwOut{Processing sequence $I = ( I_0, I_1, \cdots, I_T )$.}

Initialize $\tilde{o}_i \gets \ell$ for all $i \in [n]$.

\While{there exist unfinished requests}{
    Let $R_t$ be the queue of waiting prompts at time $t$. Let $S_t$ be the set of tokens currently processing.

    For any $i \in S_t$ with more than $\ell$ tokens but not yet finished, remove $i$ from $S_t$, set $\tilde{o}_i \gets u$, and move $i$ to the end of $R_t$.

    Let $S_t'$ be the active set after updates.

    Starting from the front of $R_t$, successively add requests to a set $I_t$ as long as $I_t \cup S_t'$ satisfies the memory constraint, using $\tilde{o}_i$ for each $i \in I_t$.

    Activate and process the requests in $I_t$; update $R_{t+1} = R_t \setminus I_t$.
}
\end{algorithm}

\proof{Proof of Theorem \ref{thm:Al_CR}}
Under $\mathcal{A}_\ell$, the request of length $\ell$ is executed first and thus waits exactly $\ell$ batches; its average output token is generated at time $\tfrac12\ell$.  
The request of length $u$ is delayed by all $\ell$ tokens of the first request, in addition to its own $u$ tokens, resulting in an average completion time of $\ell+u/2$.  
Multiplying these average sojourn times by the per-batch workload $(x_\ell+x_u)$ and then by the corresponding request sizes gives:
\[
\mathrm{TEL}(\mathbf{o};\mathcal{A}_\ell)
  =\bigl[\ell(\tfrac12\ell)(x_\ell+x_u)\bigr]x_\ell
   +\bigl[(\ell+\tfrac12 u)(x_\ell+x_u)\bigr]x_u
  =\frac{u^{2}}{2}\,
     \Bigl[\alpha^{2}x_\ell^{2}
           +3\alpha^{2}x_\ell x_u
           +(2\alpha^{2}+1)x_u^{2}\Bigr],\tag{1}
\]
where $\alpha:=\ell/u\in(0,1)$.

In comparison, the Hindsight-Shortest First benchmark (\HS) schedules every token as soon as capacity permits:
\[
\mathrm{TEL}(\mathbf{o};\HS)
  =\frac{u^{2}}{2}\,
     \Bigl[\alpha^{2}x_\ell^{2}
           +2\alpha^{2}x_\ell x_u
           +x_u^{2}\Bigr].\tag{2}
\]

Taking the ratio gives:
\[
\text{CR}(\mathcal{A}_\ell)
  =1+\frac{\alpha^{2}x_\ell x_u+2\alpha^{2}x_u^{2}}
           {\alpha^{2}x_\ell^{2}+2\alpha^{2}x_\ell x_u+x_u^{2}}
    +\mathcal{O}\!\bigl(\tfrac1M\bigr).\tag{3}
\]

To bound the fractional term, let $t:=x_\ell/x_u\ge0$ and define:
\[
h(t)=\frac{\alpha^{2}t+2\alpha^{2}}{\alpha^{2}t^{2}+2\alpha^{2}t+1}.
\]
For $0\le\alpha\le\frac12$, the inequality
$\alpha^{2}t^{2}+2\alpha^{2}t+1\ge2\alpha(1+\alpha)t$ implies
$h(t)\le\alpha/[2(1-\alpha)]$ for all $t\ge0$.
For $\frac12\le\alpha\le1$, $h(t)$ is maximized at $t=0$, yielding $h(t)\le2\alpha^{2}$.
Substituting these bounds produces the desired piecewise estimate:
\[
\text{CR}(\mathcal{A}_\ell)
  \le
  \begin{cases}
     1+\dfrac{\alpha}{2(1-\alpha)}, & 0\le\alpha\le\frac12,\\[8pt]
     1+2\alpha^{2}, & \dfrac12\le\alpha\le1,
  \end{cases}
  +\mathcal{O}\!\bigl(\tfrac1M\bigr),
\]
completing the proof.
\Halmos
\endproof

\proof{Proof of Theorem \ref{cor:LG}}
We use the result of Theorem~\ref{thm:Amin} to compute the denominator:
\[
\sum_{k=1}^{\infty} k^3 q^{k-1} \left(2k q^{k-1} + \sum_{i=k+1}^{\infty} i q^{i-1} \right) = \frac{1 + 2q + 11q^2 + 8q^3 + 11q^4 + 2q^5 + q^6}{(1 - q)^6 (1 + q)^4},
\]
and the numerator:
\[
\sum_{k=1}^{\infty} \left( \sum_{i=k}^{\infty} i^2 q^{i-1} \right) 
\left( \sum_{i=k}^{\infty} q^{i-1} + 2 \sum_{i=k+1}^{\infty} (i - k) q^{i-1} \right) = \frac{1 + 3q + 6q^2 + 3q^3 + q^4}{(1 - q)^6 (1 + q)^2}.
\]
Thus, we have
\[
\text{CR}(\mathcal{A}_{\min}) = \frac{(1+q)^2 (1 + 3q + 6q^2 + 3q^3 + q^4)}{1 + 2q + 11q^2 + 8q^3 + 11q^4 + 2q^5 + q^6} \le \frac{14}{9} \approx 1.56.
\]
\Halmos
\endproof

\end{document}